\documentclass{article}

\PassOptionsToPackage{numbers, compress}{natbib}

     \usepackage[final]{neurips_2019}
     \usepackage{neurips_2019}

\usepackage[utf8]{inputenc} 
\usepackage[T1]{fontenc}    
\usepackage{hyperref}       
\usepackage{url}            
\usepackage{booktabs}       
\usepackage{amsfonts}       
\usepackage{nicefrac}       
\usepackage{microtype}      

\usepackage{color}
\usepackage{array}
\usepackage{commath}
\usepackage{latexsym}
\usepackage{mathtools}
\usepackage{amsmath}
\usepackage{amsthm}
\usepackage{amssymb}
\usepackage{graphicx}
\usepackage{wrapfig}
\usepackage{enumitem}
\usepackage{subfig}

\theoremstyle{definition}
\newtheorem{example}{\protect\examplename}
\theoremstyle{definition}
\newtheorem{defn}{\protect\definitionname}
\theoremstyle{plain}
\newtheorem{lem}{\protect\lemmaname}
\theoremstyle{plain}
\newtheorem{thm}{\protect\theoremname}
\theoremstyle{plain}

\theoremstyle{plain}
\newtheorem{claim}{Claim}

\providecommand{\definitionname}{Definition}
\providecommand{\examplename}{Example}
\providecommand{\lemmaname}{Lemma}
\providecommand{\propositionname}{Proposition}
\providecommand{\theoremname}{Theorem}

\providecommand{\tabularnewline}{\\}

\global\long\def\mat#1{\boldsymbol{#1}}

\global\long\def\optr{\mbox{tr}}

\global\long\def\opdiag{\mbox{diag}}

\global\long\def\t#1{\widetilde{#1}}
\global\long\def\h#1{\widehat{#1}}
\global\long\def\abs#1{\left\lvert #1\right\rvert }
\global\long\def\norm#1{\lVert#1\rVert}

\global\long\def\set#1{\left\{  #1\right\}  }
\global\long\def\bydef{\overset{\text{def}}{=}}

\global\long\def\EE{\mathbb{E}\,}
\global\long\def\EEk{\mathbb{E}_{k}\,}

\global\long\def\R{\mathbb{R}}
\global\long\def\E{\mathbb{E}}

\global\long\def\va{\boldsymbol{a}}

\global\long\def\vc{\boldsymbol{c}}

\global\long\def\vg{\boldsymbol{g}}

\global\long\def\vq{\boldsymbol{q}}
\global\long\def\vr{\boldsymbol{r}}

\global\long\def\vs{\boldsymbol{s}}

\global\long\def\vu{\boldsymbol{u}}
\global\long\def\vv{\boldsymbol{v}}
\global\long\def\vw{\boldsymbol{w}}
\global\long\def\vx{\boldsymbol{x}}
\global\long\def\vy{\boldsymbol{y}}
\global\long\def\vz{\boldsymbol{z}}

\global\long\def\mA{\boldsymbol{A}}

\global\long\def\mI{\boldsymbol{I}}
\global\long\def\mJ{\boldsymbol{J}}

\global\long\def\mL{\boldsymbol{L}}
\global\long\def\mM{\boldsymbol{M}}

\global\long\def\mP{\boldsymbol{P}}

\global\long\def\mR{\boldsymbol{R}}
\global\long\def\mS{\boldsymbol{S}}

\global\long\def\mU{\boldsymbol{U}}

\global\long\def\mV{\boldsymbol{V}}

\global\long\def\mX{\boldsymbol{X}}

\global\long\def\calG{\mathcal{G}}
\global\long\def\calD{\mathcal{D}}

\global\long\def\a{\alpha}

\global\long\def\e{\epsilon}

\global\long\def\T{\top}

\global\long\def\nt{\left\lfloor nt\right\rfloor }

\global\long\def\tvy{\t{\boldsymbol{y}}}

\global\long\def\ttau{\t{\tau}}
\global\long\def\tf{\t f}

\global\long\def\tva{\t{\boldsymbol{a}}}
\global\long\def\tvc{\t{\boldsymbol{c}}}

\global\long\def\hw{\widehat{w}}
\global\long\def\hvu{\widehat{\vu}}
\global\long\def\hvv{\widehat{\vv}}

\title{A Solvable High-Dimensional Model of GAN}

\author{
  Chuang Wang$^{1,2}$\\
     \texttt{ wangchuang@ia.ac.cn}
  \And
   Hong Hu$^2$ \\
    \texttt{ honghu@g.harvard.edu}
   \And
    Yue M. Lu$^2$
   \\
\texttt{ yuelu@seas.harvard.edu} 
   \AND
   \\
     1. State Key Laboratory of Pattern Recognition, Institute of Automation,\\
     Chinese Academy of Science,
  95 Zhong Guan Cun Dong Lu, Beijing 100190, China\\
  2. John A. Paulson School of Engineering and Applied Sciences, Harvard University\\
  33 Oxford Street, Cambridge, MA 02138, USA 
}
\begin{document}

\maketitle

\begin{abstract}
We present a theoretical analysis of the training process for a single-layer GAN fed by high-dimensional input data. The training dynamics of the proposed model at both microscopic and macroscopic scales can be exactly analyzed in the high-dimensional limit. In particular, we prove that the macroscopic quantities measuring the quality of the training process converge to a deterministic process characterized by an ordinary differential equation (ODE), whereas the microscopic states containing all the detailed weights remain stochastic, whose dynamics can be described by a stochastic differential equation (SDE). This analysis provides a new perspective different from recent analyses in the limit of small learning rate, where the microscopic state is always considered deterministic, and the contribution of noise is ignored. From our analysis, we show that the level of the background noise is essential to the convergence of the training process: setting the noise level too strong leads to failure of feature recovery, whereas setting the noise too weak causes oscillation.  Although this work focuses on a simple copy model of GAN, we believe the analysis methods and insights developed here would prove useful in the theoretical understanding of other variants of GANs with more advanced training algorithms.
\end{abstract}

\section{Introduction}

A generative adversarial network (GAN) \cite{Goodfellow2014a} seeks to learn a high-dimensional probability distribution from samples. 
 While there have been numerous advances on the application front \cite{Arjovsky2017,Lucic2017,Ledig2016,Isola2016,Reed2016}, considerably less is known about the underlying theory and conditions that can explain or guarantee the successful trainings of GANs.

Recently, it has been a very active area of research to study either the equilibrium properties \cite{Arora2017b,Arjovsky2017a,Feizi2017}
or the training dynamics \cite{Li2017b,Mescheder2018}. Specifically, there is a line of  works studying the dynamics of the gradient-based training algorithms \emph{e.g.},  \cite{Mescheder2018,Mescheder2017,Nagarajan2017a,Roth2017,Heusel2017,Mazumdar2019}.
The basic idea is the following. The evolution of the learnable parameters in the training dynamics can be considered as a discrete-time process.  
With a proper time scaling, this discrete-time process converges to
a deterministic continuous-time process as the learning rates tend to 0, which is characterized by an ordinary differential equation (ODE). 
By studying local stability of the ODE's fixed points, \cite{Mescheder2017} shows that oscillation in the training algorithm is due to the eigenvalues of the Jacobian of the gradient vector field with zero real part and large imaginary part. Due to this fact, various stabilization approaches are proposed, for example adding additional regularizers \cite{Nagarajan2017a,Roth2017}, and using two timescale \cite{Heusel2017} training. Very recently, \cite{Mazumdar2019} argues that those stabilization techniques may encourage the algorithms to converge non-Nash stationary points.  
All above works consider a small-learning-rates limit, where the limiting process is always deterministic. The stochasticity and the effect of the noise is essentially ignored, which may not reflect practical situations.
Thus, a new analysis paradigm to study the dynamics with the consideration of the intrinsic stochasticity is needed.

In this paper, we present a \emph{high-dimensional} and \emph{exactly solvable} model of GAN. Its dynamics can be precisely characterized at both macroscopic and  microscopic scales, where the former is deterministic and the latter remains stochastic.  
Interestingly, our theoretical analysis shows that injecting additional noise can stabilize the training.
Specifically, our main technical contributions are twofold:
\begin{itemize}
\item We present an asymptotically exact analysis of the training process of the proposed GAN model. Our analysis is carried out on both the \emph{macroscopic} and the \emph{microscopic} levels.  The macroscopic state measures the overall performance of the training process, whereas the microscopic state contains all the detailed weights information.
In the high-dimensional limit ($n \to \infty$), we show that the former converges to a deterministic process governed by an ordinary differential equation (ODE), whereas the latter
stays stochastic described by a stochastic differential equation (SDE).
\item 
We show that depending on the
choice of the learning rates and the strength of noise, the training process can reach either a successful, a failed,
 an oscillating, or a mode-collapsing phase. By studying the stabilities of the fixed points of the limiting ODEs, we precisely characterize when each phase takes place. The analysis reveals a condition on the learning rates and the noise strength for successful training.
We show that the level of the background noise is essential to the convergence of the training process: setting the noise level too strong (small signal-to-noise ratio) leads to failure of feature recovery, whereas setting the noise too weak (large signal-to-noise ratio) causes oscillation.
\end{itemize}

Our work builds upon a general analysis framework \cite{Wang2017c} for studying the scaling limits of high-dimensional exchangeable stochastic processes with applications to nonlinear regression problems. Similar techniques have also been used in the literature to study Monte Carlo methods \cite{Roberts1997a}, online perceptron learning \cite{Saad1995,Biehl1995}, online sparse PCA \cite{Wang2016}, subspace estimation \cite{Wang2018}, online ICA \cite{Wang2017} and more recently, the supervised learning of two-layer neural networks \cite{Mei2018a}, but to our best knowledge, 
this technique has not yet been used in analyzing GANs.

The rest of the paper is organized as follows. We present the proposed GAN model and the associated training algorithm in Section \ref{sec:settings}. 
Our main results are presented in Section \ref{sec:dyn}, where we show that the macroscopic and microscopic dynamics of the training process converge to their respective limiting processes that are characterized by an ODE and SDE, respectively. In Section \ref{sec:phase}, we analyze the stationary solutions of the limiting ODEs and precisely characterizes the long-term behaviors of the training process. We conclude in Section \ref{sec:con}.

\section{Formulations} \label{sec:settings}

In this section, we introduce the proposed  GAN model and specify the associated training algorithm.

\paragraph{Model for the real data.}
In order to establish the theoretical analysis,  we first impose a model for the probability distribution from which we draw our real data samples.
We assume that the
real data $ \vy_{k} \in \R^n $, $k=0,1,\ldots$ are drawn according to the following generative model:
\begin{equation}
\vy_{k}=\calG(\vc_{k,} \va_{k}; \mU, \eta_\text{T})\bydef\mU \vc_{k}+ \sqrt{\eta_\text{T}}\va_{k},
 \label{eq:real-G}
\end{equation}
where $ \mU \in \R^{n\times d}$ is a deterministic unknown feature matrix with $d$ features; $\vc_{k}\in \R^{d} $ is a random vector drawn from an unknown distribution $\mathcal{P}_{\vc}$;
 $ \va_{k} $  is an $n$-dimensional  random vector acting as the background noise; and $\eta_\text{T}$ is a parameter to control the strength of noise. 
Without loss of generality \footnote{If $\mU$  is not orthogonal, we can rewrite $\mU \vc$ in \eqref{eq:real-G} as $(\mU \mR)(\mR^{-1} \vc)$, where $\mR$ is a matrix that orthogonalizes and normalizes the columns of $\mU$. We can then study an equivalent system where the new feature vector is $\mR^{-1} \vc$.},
we assume $ \mU^\T\mU=\mI_d$, where $\mI_d$ is the $d\times d$ identity matrix.

This generative model, referred to as the spiked covariance model \cite{Johnstone2009} in the literature, is commonly used in the theoretical study of principal component analysis (PCA). We note that this model is not a trivial task for PCA even when $d=1$ if the variance of the noise $\va_k$ is a non-zero constant. As proved in \cite{Johnstone2009},  the best estimator can not perfectly recover the signal $\mU$ given an $\mathcal{O}(n)$ number of samples ${\vy_k}$. Thus, it is of sufficient interest to investigate whether a GAN can retrieve
informative results for the principal components in the same scaling limit. 
\paragraph{The GAN model} The GAN we are going to analyze is defined as follows. We assume that the generator  $\mathcal{G}$ has the same linear structure as the real data model \eqref{eq:real-G} given above:
\begin{equation}
\tvy_{k}=\calG( \tvc_{k}, \tva_{k}; \mV,\eta_\text{G})
 \label{eq:gen-model}
\end{equation}
but the parameters are different. Here, $ \tvy_k$ denotes a fake sample produced by the generator; $\tva_k$ is an $n$-dimensional random noise vector; the random variable $ \tvc_{k}$ is drawn from a fixed distribution $\mathcal{P}_{ \tvc}$; $\eta_\text{G}$ is the noise strength; and the matrix $ \mV\in \R^{n\times d}$  represents the parameters of the generator. (In an ideal case in which the generator learns the underlying true probability distribution perfectly, we have $\mV = \mU$.) Throughout the paper, we follow the notational convention that all the symbols that are decorated with a tilde (\emph{e.g.}, $ \tvy_{k}$, $ \tvc_{k}$,
$ \tva_{k}$) denote quantities associated with the generator.

We define the discriminator $\mathcal{D}$ of our GAN model as
\[
\calD( \vy; \vw)\bydef \h{D}( \vy^{ \T} \vw).
\]
Here, $ \vy$ is an input vector, which can be either the real data $ \vy_{k}$ from (\ref{eq:real-G})
or the fake one $ \tvy_{k}$ from \eqref{eq:gen-model}; $ \h{D}: \R \mapsto \R$ can be any function; and the vector $ \vw \in \R^{n}$
represents the parameters associated with the discriminator.  
Later, we will show that the generator can learn multiple features even though the discriminator only has one feature vector $\vw$. Discriminators with multiple features can also be analyzed in a similar way, but in this paper we consider the single-feature discriminator for simplicity.

\paragraph{The training algorithm.} The proposed GAN model has two set of parameters $\mV$ and $\vw$  to be learned from the data. The training process is formulated as the following MinMax problem 
\begin{equation}
\begin{aligned}
 \min_{ \mV} \max_{ \vw}\,
 \E_{ \vy \sim \mathcal{P}( \vy; \mU)} \E_{\t \vy \sim \t{\mathcal{P}}(\tvy,\mV)}\ 
 \mathcal{L}(\vy, \tvy; \vw),
 \end{aligned}
\label{eq:minmax}
\end{equation}
where the two probability distributions $\mathcal{P}( \vy; \mU)$ and $\t{\mathcal{P}}( \tvy; \mV)$ represent the distributions of the real data $ \vy$ and the fake data $\tvy$
as specified by \eqref{eq:real-G} and $\eqref{eq:gen-model}$ respectively,
and 
\begin{equation}\label{eq:def-obj}
\begin{aligned}
 \mathcal{L}(\vy, \tvy; \vw)
 \bydef 
 &
 F( \h{D}(  \vy^\T \vw))
 -  \t F(  \h D( \tvy^\T \vw)) 
  - \tfrac{ \lambda}{2} H( \vw^\T \vw)+ \tfrac{ \lambda}{2}  \optr \big ( H( \mV^\T \mV ) \big)
\end{aligned}
\end{equation}
with $F(\cdot)$ and $ \t F(\cdot)$ being two functions that quantify the performance of the discriminator and $\lambda > 0$ being a constant. The function $H(\cdot)$  acts as a regularization term introduced to control the magnitude of the parameters $\vw$ and $\mV$. It can be an arbitrary real-valued function, which is applied element-wisely if the input is a matrix.

We consider a standard training algorithm that uses the vanilla stochastic
gradient descent/ascent (SGDA) to seek a solution of \eqref{eq:minmax}. 
To simplify the theoretical analysis, we consider an online (\emph{i.e.}, streaming) setting where each data sample $\vy_k$
is used only once. 
At step $k$, the model parameters $\vw_k$ and $\mV_k$ are updated using 
a new real sample $\vy_k$ and two fake samples $\tvy_{2k}$ and $\tvy_{2k+1}$, according to
\begin{equation} \label{eq:sgd}
\begin{aligned}
\vw_{k+1}  &= \vw_{k}+ \tfrac{\tau}{n}  \nabla_{\vw_k} \mathcal{L}(\vy_k, \tvy_{2k};\vw_k)
\\
\mV_{k+1}  &=  \mV_{k}-\tfrac{\t \tau}{n} \nabla_{\mV_k} \mathcal{L}\big(\vy_k,\mathcal{G}( \tvc_{2k+1},\tva_{2k+1};\mV_k;\eta_\text{G}); \vw_k\big),
\end{aligned}
\end{equation}
where $\tvc_{2k+1}, \t \va_{2k+1}$ are random variables that generates the fake sample
 $\t \vy_{2k+1}$ according to \eqref{eq:gen-model}.
The two parameters $ \tau$ and $ \t{ \tau}$ in the above expressions control
the learning rates of the discriminator and the generator, respectively. In \eqref{eq:sgd},
we only consider a single-step update for $ \vw_{k}$. This is a special
case of Algorithm 1 in \cite{Goodfellow2014a} with the batch-size $m$ set to 1. We note that the analysis presented in this paper can be naturally extended to the mini-batch case where $m$ is a finite number.


\begin{example} \label{ex:PCA}
We define $F( \h{D}(x))= \t F( \h D(x))= x^{2}/2$, and the regularizer function $H(\mat{\mA})=  \log \cosh(\mA-\mI)$, where $\mI$ is the identity matrix with the same dimension of $\mA$, and the function $\log \cosh(\cdot)$ transforms the input matrix element-wisely. We use this specific regularizer to control the magnitude of the model parameters $\mV$ and $\vw$. In practice, any convex function with its minimum reached at zero would be fine. Our choice $\log \cosh(\mA-\mI)$ here is 
is just a convenient special case since its derivative $H^\prime (x) = \tanh(x)$ is smooth and bounded.
Furthermore, we set the regularization parameter $\lambda \to \infty$, the original problem \eqref{eq:minmax} becomes a constrained MinMax problem
\[
\min_{\opdiag(\mV^\T \mV)=\mI_d} \max_{\norm{\vw}=1} \E_{ \vy \sim \mathcal{P}} \E_{\t \vy \sim \t{\mathcal{P}}} \left[(\vy^\T \vw)^2 - (\t \vy^\T \vw)^2\right],
\] 
in which the diagonal operation $\opdiag(\mA)$ returns a matrix where the diagonal entries are the same as $\mA$ and the off-diagonal entries are all zero. The condition $\opdiag(\mV^\T \mV)=\mI_d$  ensures  that each column vector of $\mV$ is normalized. \end{example}

\section{Dynamics of the GAN} \label{sec:dyn}
\begin{defn} 
Let $ \mX_{k} \bydef [ \mU, \mV_{k}, \vw_{k}] \in \R^{n \times (2d+1)}$. We call $ \mX_{k}$ the \emph{microscopic
state} of the training process at iteration step $k$.
\end{defn}
The microscopic state $ \mX_{k}$ contains all the information about
the training process. In fact, the sequence $ \{ \mX_{k} \}_{k=0,1,2, \ldots}$ forms a Markov
chain on $ \R^{n \times (2d+1)}$. This can be easily verified from the update rule of $ \mX_{k}$ as defined in \eqref{eq:sgd}, in which the real data
$ \vy_{k}$ and  fake data $ \tvy_{k}$  are drawn according to   \eqref{eq:real-G} and \eqref{eq:gen-model}
respectively. The Markov chain is driven by the initial state $ \mX_{0}$
and the sequence of random variables $ \{(\vc_{k}, \va_{k}, \tvc_{2k}, \tva_{2k}, \tvc_{2k+1}, \tva_{2k+1}) \}_{k=0,1,2,\ldots}$. 
\begin{defn}
Let $\mP_k\bydef\mU^\T \mV_k$, $\vq_k\bydef\mU^\T\vw_k$, $\vr_k\bydef\mV_k^\T\vw_k$, $\mS_k\bydef\mV_k^\T \mV_k$, and $z_k\bydef\vw_k^\T \vw_k$. We call the tuple $\{ \mP_k, \vq_k, \vr_k, \mS_k, z_k \}$
the\emph{ macroscopic
state} of the Markov chain $ \mX_{k} $ at step $k$. 
\end{defn}
Those macroscopic quantities measure the cosine similarities among the feature vectors of the true model $\mU$, the generator $\mV_k$ and the discriminator $\vw_k$. For example,  the cosine of the angle between the $i$th true feature ({\em i.e.}, the $i$th column of $\mU$) and the $j$th feature estimated in the generator ({\em i.e.},  the $j$th column of $\mV_k$) is 
$[\mP_k]_{i,j}/\sqrt{[\mS_k]_{j,j}}$, where $[\mP_k]_{i,j}$ is the inner product between the two feature vectors and  $\sqrt{[\mS_k]_{j,j}}$ is the norm of the $j$th column of $\mV_k$. (The columns of $\mU$ are unit vectors and need not be normalized here.) For simplicity, we introduce a compact notation for the macroscopic state:  \begin{equation} \label{eq:M}
\mM_{k}\bydef \mX_k^\T \mX_k = \begin{bmatrix}\mI & \mP_{k} & \vq_{k} \\
\mP_{k}^\T & \mS_{k} & \vr_{k} \\
\vq_{k}^\T & \vr^\T_{k} & z_{k}
\end{bmatrix}.
\end{equation}

%
%

In what follows, we  investigate  the dynamics of the training algorithm
 \eqref{eq:sgd} at both the macroscopic and the microscopic levels.
At the macroscopic level, by examining the cosines of the angles, we study how closely the model parameters
$ \mV_{k}$, $ \vw_{k}$ associated with the generator and discriminator
 can align with the ground truth feature vectors, {\em i.e.,} the columns of $ \mU$. At
the microscopic level, we study how the elements in the matrix $ \mV_{k}$ and the vector
$ \vw_{k}$ evolve as a stochastic process. As our analysis will reveal, the mechanisms behind
the two levels are different: the macroscopic dynamics is asymptotically
deterministic whereas the microscopic dynamics stays stochastic even as
$n \to \infty$.

\subsection{Macroscopic dynamics}

We first study the asymptotic dynamics of the macroscopic state $ \mM_{k}$. 
Our theoretical analysis is carried out under the following assumptions.
\begin{enumerate}[label={(A.\arabic*)}]
\item \label{ass:c} 
The sequences of  $\vc_{k}\sim \mathcal{P}_{\vc}$ and $\tvc_{k}\sim \mathcal{P}_{\tvc}$ for $k=0,1,\ldots$ are i.i.d. random variables with  bounded moments of all orders, and  $\{\vc_{k}\}$ is independent of
$\{\tvc_k\}$.
\item \label{ass:a} 
The sequences $\set{\va_{k}}$ and $\set{\tva_{k}}$ for $k=0,1,\ldots$ are both independent Gaussian vectors with zero mean and the covariance matrix $ \mI_n$. Moreover, $\{\va_{k}\}$, $\{\tva_{k}\}$ are independent of  $\{\vc_{k}\}$ and $\{\tvc_k\}$.
\item  \label{ass:dif}
The first-order derivative of $H(\cdot)$ and the derivatives up to fourth order of the functions  $F(\h D(\cdot))$ and $\t F(\h D(\cdot))$  exist and they are also uniformly bounded.
\item  \label{ass:m4}
Let  $ [\mU, \mV_{0}, \vw_{0}]$ be the initial microscopic state. For $i=1,2,\ldots,n$, we have $ \EE [ \sum_{\ell=1}^d ( [\mU]_{i,\ell}^4+[\mV_0]_{i,\ell}^4+[\vw_0]_i^4] ) \leq C/n^2$, where $C$ is a constant not depending on $n$.
\item \label{ass:init}
The initial macroscopic state $\mM_0$ satisfies $\EE \norm{\mM_0 - \mM_0^\ast}\leq C/\sqrt{n}$, where $\mM_0^\ast$ is a deterministic matrix and $C$ is a constant not depending on $n$.
\end{enumerate}

We provide a few remarks on the above assumptions. In Assumption \ref{ass:c}, $\mathcal{P}_{\vc}$ and $\mathcal{P}_{\tvc}$ can be different. For example, $\vc$ is Gaussian, and $\tvc$ is uniform on $[-1, 1]^d$. The assumption \ref{ass:a} can be relaxed to non-Gaussian cases as long as all moments of ${\va_k}$ and $\tva_k$ are bounded, but we use Gaussian assumption here to simplify the proof.
The assumption \ref{ass:m4} requires that the elements in the parameter matrix of real data $\mU$ and initial microscopic state $\mX_0$ are $\mathcal{O}(1/\sqrt{n})$ numbers. Intuitively, this assumption ensures that $\mU$ and $\mX_0$ are generic matrices with $\mathcal{O}(1)$ Frobenius norms ({\em i.e.,} not the matrices that most elements are zeros and only few elements are large numbers). 
The assumption \ref{ass:init} ensures that the initial macroscopic states converges to a deterministic value as the system size $n$ goes to infinity. The following theorem proves that if the initial state is convergent, then the whole training process converges to a deterministic process as $n\to\infty$, which is characterized by an ODE.
\begin{thm}
\label{thm:ODE} Fix $T > 0$. It holds under Assumptions \ref{ass:c}--\ref{ass:init} that
\begin{equation} \label{eq:conv}
\max_{0 \le k \le nT} \EE \big\|\mM_k - \mM \big(\tfrac{k}{n}\big) \big\| \leq \tfrac{C(T)}{ \sqrt{n}},
\end{equation}
where $C(T)$ is a constant that depends on $T$ but not on $n$, and 
$\mM(t)=\begin{bmatrix}\mI & \mP_{t} & \vq_{t} \\
\mP_{t}^\T & \mS_{t} & \vr_{t} \\
\vq_{t}^\T & \vr_{t}^\T & z_{t}
\end{bmatrix} \in\R^{(2d+1)\times(2d+1)}
$ is a deterministic function. Moreover, $\mM(t)$ is the unique solution of the following ODE:
\begin{equation} \label{eq:ODE}
\begin{aligned}
\tfrac{\dif}{\dif t} \mP_t &=
\t\tau \big( \vq_t \t\vg_t^\T + \mP_t \mL_t  \big)
\\
\tfrac{\dif}{\dif t} \vq_t &= \tau \big( \vg_t - \mP_t \t\vg_t+ \vq_t h_t \big)
\\
\tfrac{\dif}{\dif t} \vr_t &=\tau \big( \mP_t^T \vg_t - \mS_t \t\vg_t + \vr_t h_t \big)
 + \t\tau \big( z_t \t\vg_t + \mL_t\vr_t  \big)
\\
\tfrac{\dif}{\dif t} \mS_t &= \t\tau \big( \vr_t \t\vg^\T_t + \t\vg_t \vr_t^\T 
 + \mS_t \mL_t + \mL_t\mS_t\big)
\\
\tfrac{\dif}{\dif t} z_t &= 2\tau ( \vq_t^\T \vg_t - \vr_t^\T \t\vg_t  + z_t h_t ) + \tau^2 b_t
\end{aligned}
\end{equation}
 with the initial condition $ \mM(0)=\mM_0^\ast$, where
\begin{equation}\label{eq:def-ggb}
\begin{aligned}
\vg_t&=\big\langle \vc f(\vc^\T \vq_t + e \sqrt{z_t\eta_{\text{T}}} )\big\rangle_{\vc,e}
,\;
\t\vg_t=\big\langle \tvc \t f(\tvc^\T \vr_t +  e \sqrt{z_t\eta_{\text{G}}})\big\rangle_{\tvc, e}
,\;
\mL_t= -\lambda \opdiag (H^\prime(\mS_t))
\\
h_t&=\big\langle  f^\prime(\vc^\T \vq_t + e \sqrt{z_t\eta_{\text{T}}} )\big\rangle_{\vc,e}
-\big\langle \t f^\prime(\tvc^\T \vr_t +  e \sqrt{z_t\eta_{\text{G}}})\big\rangle_{\tvc, e}-\lambda H^\prime(z_t),
\\
b_t&=\eta_{\text{T}}\big\langle  f^2(\vc^\T \vq_t + e \sqrt{z_t\eta_{\text{T}}} )\big\rangle_{\vc,e}
+\eta_{\text{G}}\big\langle \t f^2(\tvc^\T \vr_t + e \sqrt{z_t\eta_{\text{G}}})\big\rangle_{\tvc,e}.
\end{aligned}
\end{equation}
The two functions $f$, $\tf$ stand for $f(x)= \frac{\dif}{\dif x}F( \h D(x))$
and $\tf(x)= \frac{\dif}{\dif x} \t F( \h D(x))$, and $f^{ \prime}$, $\tf^\prime$ and $H^\prime$ are derivatives
of $f$, $\tf$ and $H$ respectively. The two constants $\eta_{\text{T}}$ and $\eta_{\text{G}}$ are the strength of the noise in the true data model and the generator, respectively.
The brackets $ \left \langle \cdot \right \rangle _{\vc,e}$
and $ \left \langle \cdot \right \rangle _{ \tvc,e}$ denote the averages
over the random variables $\vc \sim \mathcal{P}_{\vc}$, $ \tvc \sim \mathcal{P}_{ \tvc}$, and $e \sim \mathcal{N}(0,1)$,
where $\mathcal{P}_{\vc}$ and $\mathcal{P}_{ \tvc}$ are the distributions involved
in defining the generative model (1) and the generator (2). 
\end{thm}
This theorem implies that for each $k = \lfloor t n \rfloor$ for some $t \in [0, T]$,  the macroscopic state $ \mM_{k}$  converges  to a deterministic number  $\mM(t)$, and the convergence rate is $\mathcal{O}(1/\sqrt{n})$. The limiting ODE \eqref{eq:ODE} for the macroscopic states involves  $\mathcal{O}(d^2)$ variables, where $d$ is the number of internal features often assumed to be a finite number that is much less than $n$. This ODE is essentially different from the ODE derived in the small-learning-rate limit \cite{Mescheder2018,Mescheder2017,Nagarajan2017a,Roth2017,Heusel2017,Mazumdar2019}, in which the number of variables is $\mathcal{O}(n)$. 

The complete proof can be found in the Supplementary Materials. We briefly sketch the proof here.
First, we note that $\mM_k$ is a discrete-time stochastic process driven by the Markov chain $\mX_k$.
Then, we apply the martingale decomposition for $\mM_k$ and get
\begin{equation*}
\mM_{k+1}-\mM_{k}=\tfrac{1}{n}\phi(\mM_{k})+(\mM_{k+1}-\EEk\mM_{k+1})+[\EEk\mM_{k+1}-\mM_{k}-\tfrac{1}{n}\phi(\mM_{k})],
\end{equation*}
where the matrix-valued function $\phi(\mM)$ represents the functions on the right hand sides of the ODE \eqref{eq:ODE},
and $\EEk$ denotes the conditional expectation given the state of
the Markov chain $\mX_{k}$.
Finally, we show the martingale $\sum_{k^{\prime}=0}^{k}(\mM_{k^{\prime}+1}-\EE_{k^{\prime}}\mM_{k^{\prime}})$ and the higher-order term $\EEk\mM_{k+1}-\mM_{k}-\tfrac{1}{n}\phi(\mM_{k})$ have no contribution when $n$ goes to infinity.

 Due to the limitation of our current proof,  the constant $C(T)$ in \eqref{eq:conv} grows exponentially as $T$ increases. This is not a problem for any finite $T$, but may cause some problem to study the long time behavior when $T\to\infty$. 
However, if we impose a sufficient large regularizer parameter $\lambda$ to limit the norms of the microscopic weights $\mV_k$ and $\vw_k$,  then the macroscopic state $\mM_k$ is bounded as $[\mM_k]_{i,j}^2 \leq [\mM_k]_{i,i}  [\mM_k]_{j,j}$. In our experiments, $\lambda > 1$ is sufficient. In this case, the constant $C(T)$ is bounded not depending on $T$. In Example 1, when $\lambda \to \infty$, $[\mM_k]_{i,i}=1$, and therefore $[\mM_k]_{i,j}^2 \leq 1$ and $C(T)\leq (2d+1)^2$, where the number of features $d$ is considered a constant not growing with $n$. This justifies the fixed points analysis of the ODE as discussed in Section~\ref{sec:phase}, which reflects the long-time training behavior. A better proof strategy to get rid of this dependence of $T$ is also possible, {\em e.g.}, \cite{Jourdain2014}.


\paragraph{Numerical verification.}
 We verify the theoretical prediction given by the ODE (\ref{eq:ODE}) via numerical simulations under the settings stated in Example \ref{ex:PCA}. The results are shown in Figure \ref{fig:macro-dyn-d2}. The number of features is $d=2$, and $\vc_k$ and $\tvc_k$ are both Gaussian with zero mean and covariance $\opdiag([5,3])$. The dimension is
 $n=5,000$, and the learning rates of the generator and discriminator are $\ttau=0.04$ and $\tau=0.2$ respectively. After testing different noise strength $\eta_{\text{T}}=\eta_{\text{G}}=2,1,4$, we have observed at least three nontrivial dynamical patterns: success, oscillating or mode collapsing.
 In all these experiments, our theoretical predictions match the actual trajectories of the macroscopic states pretty well. 
 
Let us take a closer look at the successful case as shown in the left figure in Figure \ref{fig:macro-dyn-d2}. The dynamics can be split into 4 stages. At the first stage, the discriminator learns the first feature of the true model. At this state, $[\vq_t]_1$ quickly increases. At the second stage, the generator starts to learn the first feature and the discriminator is deceived. At this stage, $[\mP_t]_{1,1}^2$ increases and $[\vq_t]_1^2$ decreases. Once the discriminator completely forgets the first feature as $[\vq_t]_1\approx0$, the third state begins. The discriminator starts to learn the second feature as $[\vq_t]_2^2$ increases. Then, at the last stage, the generator learns the second feature and the discriminator is fooled again. In this region, $[\mP_t]^2_{2,2}$ increases and $[\vq_t]_2^2$ decreases down to 0. Eventually, the generators learns both features and the discriminator is completely fooled. 
It ends up at a stationary state that  $\vq_t=\mat{0}$ and $\mP_t$ is nearly an identity matrix.
Interestingly, this experiment shows that the generator learn features sequentially given a single-feature discriminator. This may be a reason why in practice, the discriminator's structure can be much simpler than the generator's. 
 
\begin{figure*}[t]
\center 
\includegraphics[scale=0.582]{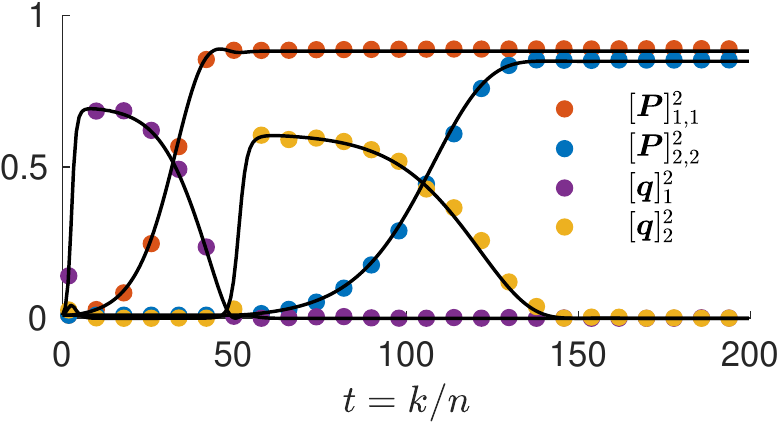}
\includegraphics[scale=0.582]{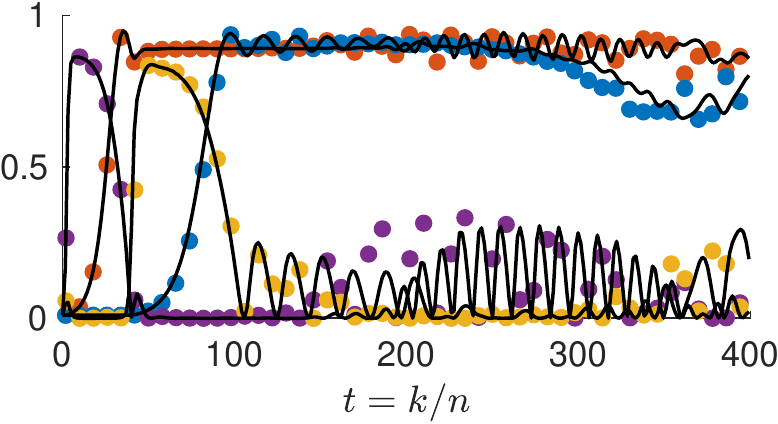}
\includegraphics[scale=0.582]{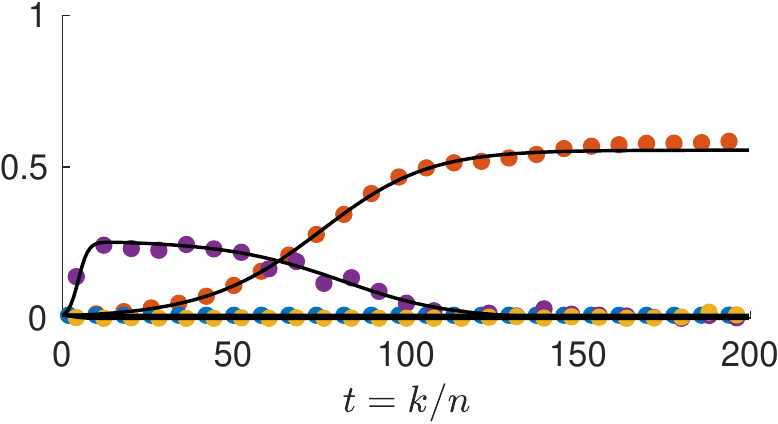}
\caption{\label{fig:macro-dyn-d2} Macroscopic dynamics of the GAN with $d=2$ features: $[\mP_k]_{i,j}$ is the cosine of the angle between $i$'th column vector of the real feature matrix $\mU_k$ and $j$'th column vector of the generator's weight matrix $\mV_k$. Similarly, $[\vq_k]_{i}$ is the cosine of angle between $i$'th column vector of $\mU_k$ and the discriminator's weight vector $\vw_k$. Colored dots are results from experiments, and the curves tracing these dots are our theoretical prediction by the ODE \eqref{eq:ODE}. 
From the left to right, the variance of background noise is $\eta_{\text{T}}=\eta_{\text{G}}=2,1,4$ respectively, and other parameters are the same. The left figure is an example of successful training, where two features (red and blue dots) are retrieved by the generator. The center figure shows an oscillating training. It happens when noise are weak. The right figures shows a mode collapsing state, in which only the first feature  are  estimated by the generator.  
}
\end{figure*}

\subsection{Microscopic dynamics}

\label{sec:micro}

In this section, we study how the elements in $ \mX_{k}=[ \mU, \mV_{k}, \vw_{k}]$
evolve during the training process. 
Instead of studying the trajectory of $\mX_k$, we study the evolution of the \emph{empirical
measure} of the microscopic states, which is defined as
\begin{equation*}
\begin{aligned}
\mu_{k}(\hvu,& \hvv, \hw) 
\bydef 
\tfrac{1}{n} { \textstyle\sum_{i=1}^{n}} 
\delta\big(
\big[ \hvu^\T, \hvv^\T, \hw \big] 
- \sqrt{n} \big[ [\mU]_{i,:}, [\mV_k]_{i,:}, [\vw]_i  \big]
\big)
\\
\end{aligned}
\end{equation*}
where  $ \delta( \cdot)$
is a Dirac  measure on $\R^{2d+1}$ and $[\mU]_{i,:}, [\mV_k]_{i,:}$ are $i$th row of $\mU$ and $\mV_k$ respectively.  The scaling factor $\sqrt{n}$ in the Dirac measures is introduced because $[\mU]_{i,\ell}$, $[\mV_k]_{i,\ell}$ and $[w_{k,}]_i$ are $ \mathcal{O}(1/\sqrt{n})$ quantities. 

We next embed the discrete-time measure-valued
stochastic process $ \mu_{k}$ into a continuous-time process by defining
$
\mu_{t}^{(n)} \bydef \mu_{k}(\hvu, \hvv,\hw)   \; \text{with }k= \nt.
$
 Following the general technical approach presented in \cite{Wang2017c}, we can show
that under the same assumptions as Theorem \ref{thm:ODE}, given $T>0$,
the sequence of measure-valued process $ \{ \{ \mu_{t}^{(n)} \}_{t \in[0,T]} \}_{n}$
converges weakly to a deterministic process $ \{ \mu_{t} \}_{t \in[0,T]}$.
In addition, $\mu_t$ is the measure of the solution to the stochastic differential equation
\begin{equation} \label{eq:SDE}
\begin{aligned}
\dif \hvu_t &= 0 \\
\dif \hvv_t &=\t\tau\big( \hw_t \t\vg_t + \mL_t\hvv_t \big)\dif t\\
\dif \hw_t &= \tau \big( \hvu_t^\T\vg_t +  \hvv_t^\T \t\vg_t \big.
 +\big. \hw_t h_t \big) \dif t 
+ \tau \sqrt{b_t} \dif B_t
\end{aligned}
\end{equation}
where $(\hvu_0,\hvv_0,\hw_0)\sim\mu_0$; $B_t$ is the standard Brownian motion.
The functions $\vg_t$, $\t{\vg}_t$, $\mL_t$, $h_t$ and $b_t$ are defined in \eqref{eq:def-ggb}, in which the macroscopic quantities $\mP_t,\;\mS_t,\;\vq_t,\;z_t,\;\vr_t$ are computed as follows
\begin{equation} \label{eq:coef}
\begin{aligned}
\mP_t&= \langle \mu_t, \hvu \hvv^\T \rangle,  \quad 
 \mS_t=\langle \mu_t, \hvv \hvv^\T \rangle, 
\vq_t&=\langle \mu_t, \hvu \hw \rangle, \quad 
z_t=\langle \mu_t, \hw^2 \rangle,\quad
 \vr_t=\langle \mu_t, \hvv \hw \rangle,
\end{aligned}
\end{equation}
where $\langle \mu_t, \cdot \rangle$ denotes the expectation with respect to the measure $\mu_t$.  

The SDE \eqref{eq:SDE} shows the intuitive meaning of the functions defined in \eqref{eq:def-ggb}: $\vg_t$, $\t{\vg}_t$, $\mL_t$, $h_t$ are drift coefficients of the SDE and $b_t$ is the diffusion coefficient of the SDE. 
We also note that if one follows the analysis in the small-learning-rate limit
\cite{Mescheder2018,Mescheder2017,Nagarajan2017a,Roth2017,Heusel2017,Mazumdar2019},
one will get an ODE for the microscopic states. Compared to our SDE formula, the diffusion term $\tau \sqrt{b_t} dB_t$ is missing in those works, and therefore the effect of the noise can not be analyzed.
 

Moreover, the deterministic measure $\mu_t$ is unique solution of the following PDE (given in its weak form):  for
any bounded smooth test function $  \varphi(\hvu, \hvv,\hw)$, 
\begin{equation}  \label{eq:weak}
\begin{aligned}
&\od{}{t} \big \langle \mu_t, \varphi(\hvu, \hvv,\hw) \big \rangle   =
\\
 &
\ttau\big \langle \mu_{t}, \big(
  \hw \t\vg^\T_t+ \hvv^\T\mL_t   \big)
\nabla_{\hvv} \varphi 
 \big \rangle 
 +\tau\big \langle \mu_t, 
   \big(\hvu^\T\vg_t- \hvv^\T\t\vg_t +h_t\hw \big)
   \pd{}{\hw}  \varphi \big \rangle
+\tfrac{\tau^2}{2}b_t 
\big \langle \mu_t,  \pd[2]{}{\hw} \varphi\big \rangle 
\end{aligned}
\end{equation}
where $\vq_{t}$, $\vr_t$, $\mS_t$, and  $z_{t}$ are defined in \eqref{eq:coef}, and the functions 
$\vg_t$, $\t\vg_t$, $b_t$, $h_t$ and $\mL_t$ are defined in \eqref{eq:def-ggb}.
We refer readers to \cite{Wang2017c} for a general framework for rigorously establishing the above scaling limit.

The connection between the microscopic and macroscopic dynamics  can also be derived from the weak formulation of the PDE. Let $\varphi$ being each element of $\hvu \hvv^\T,\; \hvu\hw,\; \hvv \hw,\; \hvv \hvv^\T,\; \hw^2$, and substituting those $\varphi$  into the PDE \eqref{eq:weak}, we can derive the ODE \eqref{eq:ODE}.
In the setting of this paper, the macroscopic dynamics enjoys a closed ODE: We can predict the macroscopic states without solving the PDE nor SDE at microscopic scale. However, in a more general setting, e.g. when we add a  regularizer other than the L2 type, the ODE itself may not be closed. In that case, one has to solve  the PDE directly.

\paragraph{Numerical verification.} We verify the predictions given by the PDE (\ref{eq:weak}) by setting $d=1$ using a special
choice of the $(n\times 1)$-dimensional target feature matrix $ \mU$ whose elements are all $1/\sqrt{n}$ with $n=10,000$. We also set the
initial condition $\mu_{0}(\h v, \hw |\h u=1)$ to be a Gaussian distribution.
(When $d=1$, the macroscopic quantities $P_t$, $q_t$, $r_t$, $S_t$ reduce to scalars, so we remove their boldface here.)
In this case, the PDE (\ref{eq:weak}) admits a particularly simple analytical solution: at any time $t$, the solution $\mu_{t}(\h v,\hw|\h u=1)$ is a Gaussian distribution whose mean
and covariance matrix are given by
$\E_{ \mu_{t}(\h v, \hw|\h u=1) }
\begin{bmatrix}\h v\\
\hw
\end{bmatrix}  =\begin{bmatrix}P_{t}\\
q_{t}
\end{bmatrix}, 
$
$\E_{\mu_{t}(\h v,\hw |\h u=1)}
\begin{bmatrix}\h v\\
\hw
\end{bmatrix}\begin{bmatrix}\h v & \hw\end{bmatrix}=\begin{bmatrix}S_t & r_{t}\\
r_{t} & z_{t}
\end{bmatrix}.
$
Figure \ref{fig:micro} overlays the contours of the probability distribution
$ \mu_{t}(\h v,\hw|\h u=1)$ at different times $t$ over the point
clouds of the actual experiment data $(\sqrt{n}[\vw_k]_i, \sqrt{n} [\mV_{k}]_{i,1})$. We can see that the theoretical prediction given by (\ref{eq:weak}) has excellent agreement with simulation results.

\begin{figure}
\centering{}
\includegraphics[scale=0.73]{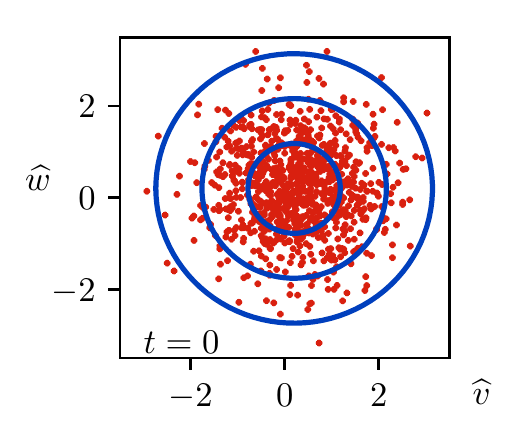}
\hspace{-1.8em}
\includegraphics[scale=0.73]{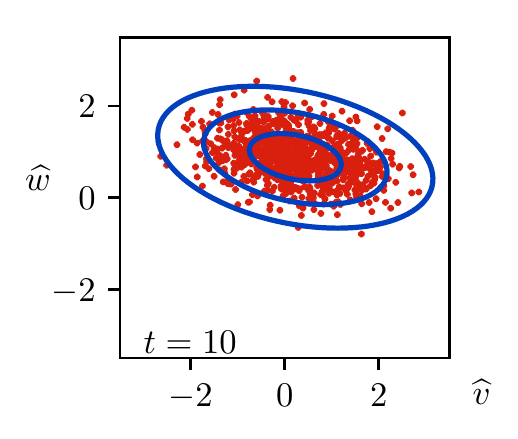} 
\hspace{-1.8em}
\includegraphics[scale=0.73]{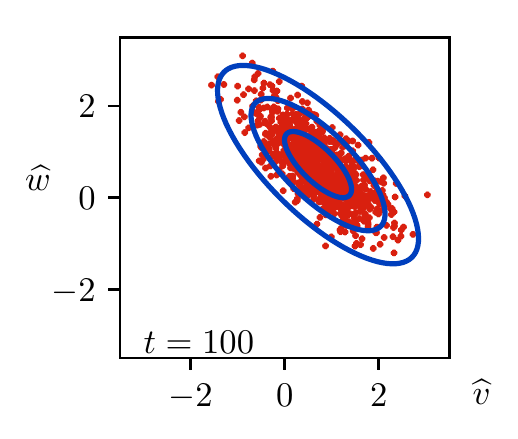} 
\hspace{-1.8em}
\includegraphics[scale=0.73]{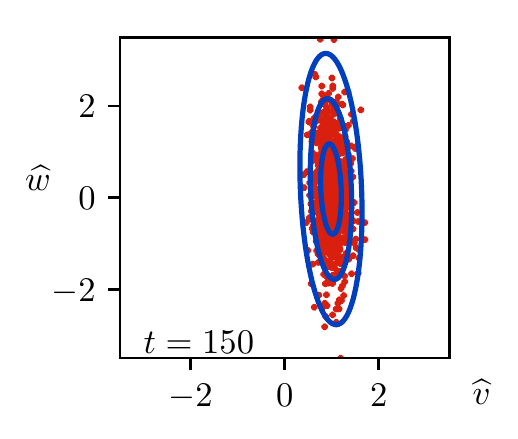}
\caption{\label{fig:micro}The evolution of the microscopic
states at $t=0$, $10$, $100$, and $150$. For each fixed $t$, the red points in the corresponding figure represent the values of $(\h v,\h w)=(\sqrt{n}[\mV_k]_{i,1}, \sqrt{n}[\vw_k]_{i})$ for $i = 1, 2, \ldots, n$, where $k=\nt$. The blue ellipses illustrate the contours corresponding to one, two, and three standard deviations of the 2-D Gaussian distribution predicted by the PDE \eqref{eq:weak}. 
}

\end{figure}

\section{Local Stability Analysis of the ODE for the Macroscopic States}\label{sec:phase}
\vspace{-0.5em}
In this section, we study how the parameters, such as  the learning rates
$ \tau$ and $ \ttau$, noise strength $\eta_{\text{G}}$ and $\eta_{\text{T}}$ affect  the training algorithm. 
We  will focus on the concrete model as described in Example \ref{ex:PCA} so that we can have analytical solutions.

In order to further reduce the degrees of freedom of the ODE \eqref{eq:ODE}, we let the regularization parameter $ \lambda \to \infty$.  In this case, the vector $\vw_{k}$ and all columns vectors of $\mV_k$ are always normalized. Thus $z_k=1$ and $[\mS]_{i,i}=1$.
The macroscopic state is then described by $\mP_k$, $\vq_k$, $\vr_k$ and off-diagonal terms of $\mS_k$.
Correspondingly, the ODE in Theorem \ref{thm:ODE} reduces to 
\begin{equation}
\begin{cases}
\frac{\dif}{\dif t} \mP_t & = \ttau \left( \vq_t \vr_t^\T\mat{\t\Lambda} + \mP_t \mL_t\right)  
 \\
\frac{\dif}{\dif t}\vq_{t} & = \tau \big( \mat{\Lambda}\vq_t - \mP_t \mat{\t\Lambda} \vr_t+ h_t\vq_t  \big) 
\\
\tfrac{\dif}{\dif t} \vr_t &=\tau \big( \mP_t^T \mat{\Lambda}\vq_t   - \mS_t \mat{\t\Lambda} \vr_t + h_t \vr_t \big)
 + \t\tau \big( \mat{\t\Lambda}  + \mL_t  \big)\vr_t
\\
\tfrac{\dif}{\dif t} \mS_t &= \t\tau \big( \vr_t \vr_t^\T\mat{\t\Lambda} + \mat{\t\Lambda} \vr_t\vr_t^\T 
 + \mS_t \mL_t + \mL_t\mS_t\big)
\end{cases} \label{eq:PCA-ODE}
\end{equation}
where 
$ \mat\Lambda$ and $ \t{ \mat\Lambda}$ are the covariance matrices of the distributions
$P_{\vc}$ and $P_{ \tvc}$, respectively; and
\begin{equation} \label{eq:h-mL}
\begin{aligned}
h_t&=  
(1 - \tfrac{\tau\eta_{G}}{2} ) \vr_t^\T \t{ \mat{\Lambda}} \vr_t 
-(1+\tfrac{\tau\eta_{T}}{2}) \vq_t^\T \mat{\Lambda} \vq_t
- \tau\tfrac{\eta_{G}^2 + \eta_{T}^2}{2}
,\quad\quad
\mL_t&=-\opdiag(\vr_t \vr_t^\T \t{\mat\Lambda}),
\end{aligned}
\end{equation}
in which $\eta_{\text{T}}$ and $\eta_{\text{G}}$ are the variance of noise in the true data model and generator, respectively.
The derivation from the ODE (\ref{eq:ODE}) to  \eqref{eq:PCA-ODE} is presented
in the Supplementary Materials.

Next, we discuss under what conditions, the GAN can reach a desirable training state by studying local stability of a particular type of fixed points of the ODE \eqref{eq:PCA-ODE}.  The perfect estimation of the generator corresponds to $\mP_t$ being an identity matrix
(up to a permutation of rows and columns).
A complete fail state relates to $\mP=\mat{0}$. 
Furthermore,  It is easy to verify that if $\vq_t=\vr_t=\mat{0}$, the ODE \eqref{eq:PCA-ODE} will be stable for any $\mP_t=\mP$. 


\begin{claim} \label{claim:stb}
The macroscopic states $\mP_t,\; \vq=\vr=\mat{0}$ for all valid $\mP_t$ are always the fixed points of the ODE  \eqref{eq:PCA-ODE}. Furthermore, a sufficient condition that the perfect estimation state $\mP_t=\mI,\; \vq=\vr=\mat{0}$ is locally stable and  the failed state $\mP_t=\mat{0},\; \vq=\vr=\mat{0}$ is  unstable   if
\begin{equation} \label{eq:con}
\tfrac{1}{2} \max_{\ell} \{ \Lambda_\ell - \t\Lambda_\ell + \a  \t\Lambda_\ell  \} \leq \tau \overline{\eta^2}< \min_\ell{ \Lambda_{\ell} },
\end{equation}
where $\a=\frac{\ttau}{\tau}$, $\overline{\eta^2}=\frac{1}{2}(\eta_{\text{T}}^2 + \eta_{\text{G}}^2) $,
and $\Lambda_\ell = [\mat{\Lambda}]_{\ell,\ell}$, $\t\Lambda_\ell = [\t{ \mat{\Lambda}}]_{\ell,\ell}$.\end{claim}
The proof can be found in the Supplementary Materials.
%
If the right inequality in \eqref{eq:con} is violated, any feature $\ell$ with the signal-to-noise ratio  $[\mat\Lambda]_{\ell,\ell}<\tau \overline{\eta^2}$ is not learned by the generator resulting {\em mode collapsing}. The right figure in Figure \ref{fig:macro-dyn-d2} demonstrates this situations, where only one of the two features is recovered. If the left inequality in \eqref{eq:con} is violated,
the training processes can be trapped in an {\em oscillation phase}.  
This phenomenon is shown in the middle figure in Figure  \ref{fig:macro-dyn-d2}.
This result indicates that proper background noise can help to avoid oscillation and stabilize the training process. In fact, the trick of injecting additional noise has been used in practice to train multi-layer GANs  \cite{sonderby2016amortised}. To our best knowledge, our paper is the first theoretical study on why noise can have such a positive effect via a dynamic perspective.

In experiments, the training is not ended at the perfect recovery point due to the presence of the noise but converges at another fixed point nearby. This is because the perfect state is marginally stable, as the Jacobian matrix  always has zero eigenvalues. It indicates that there are other locally stable fixed points near $\mP=\mI$.
In fact, all points in the hyper-rectangle region satisfying $\vq=\vr=\mat{0}$ and 
$
\abs{p^\ast_\ell} \leq \abs{[\mP]_{\ell,\ell}} \leq 1, \;\; \forall \; \ell=1,2,\ldots,d
$ 
are locally stable for some critical $p^\ast_\ell$. In the matched case when $\Lambda_\ell = \t\Lambda_\ell$, we have 
$ 
 p^\ast_\ell= 
\big[
(\Lambda_\ell - \tau \overline{\eta^2})(\t\Lambda_\ell + \tau \overline{\eta^2} - \a \t\Lambda_\ell)
/ ( \Lambda_\ell \t \Lambda_\ell) 
\big]^{1/2},
$ 
$\a=\frac{\ttau}{\tau}$ and
$\overline{\eta^2}=\frac{1}{2}(\eta^2_{\text{T}} + \eta^2_{\text{G}})$.
Starting from a point near the origin, numerical solution of the ODE shows  the training processes are  ended up at the corner of this hyper-rectangle, {\em i.e.}, $\mP^\ast=\opdiag (\{p^\ast_\ell, \; \ell=1,2,\dots,d\})$. In the small-learning rate limit $\tau \to 0$ and the learning rate ratio $\alpha \to 0$, we get the perfect recovery $\mP^\ast = \mI$. The limit $\tau \to 0$, $\alpha \to 0$ was studied in the small-learning-rate analysis with the two-time scaling \cite{Heusel2017}, and the result is consistent, but our analysis includes the situations with finite $\tau$ and $\alpha$.

In addition, we provide a phase diagram analysis in a single-feature case $d=1$ in the  Supplementary Materials. All possible fixed points in this case are enumerated and their local stability is analyzed. This helps us understand
the successful recovery condition $\eqref{eq:con}$, which is the intersection of the informative phases  that each feature can be recovered individually.



\section{Conclusion} \label{sec:con}
\vspace{-0.5em}
We present a simple high-dimensional model for GAN with an exactly analyzable training process. Using the tool of scaling limits of stochastic processes, we show that the macroscopic state associated with the training process converges to a deterministic process characterized as the unique solution of an ODE, whereas the microscopic state remains stochastic described by an SDE, whose time-varying probability measure is described by a limiting PDE. 

Indeed, it is a common picture in statistical physics that the macroscopic states of large systems tend to converge to deterministic values due to self-averaging. These notions, especially the mean-field dynamics, have been applied to analyzing neural networks both in shallow \cite{Saad1995,Biehl1995} and deep models  \cite{nguyen2019mean}. However, this mean-field regime was not considered in previous analyses of GAN. For example, a series of recent works \emph{e.g.},  \cite{Mescheder2018,Mescheder2017,Nagarajan2017a,Roth2017,Heusel2017,Mazumdar2019} considers a different scaling regime where the learning rate goes to zero but the system dimension $n$ stays fixed. In that regime, the microscopic dynamics are deterministic even with the presence of the microscopic noise. In contrast, we study the regime where the learning rate is fixed but the dimension $n\to\infty$. This setting allows us to quantify the effect of training noise in the learning dynamics.

In this paper, we only consider a linear generator with a latent variable $\tvc$ drawn from a fixed distribution $\mathcal{P}_{\tvc}$, but our analysis can be extended to a more complex non-linear model with a learnable latent-variable distribution. Specifically, in order to compute derivatives w.r.t. $\mathcal{P}_{\tvc}$, the latent variable $\tvc\sim\mathcal{P}_{\tvc}$ should be reparameterized by a deterministic function $\tvc = f(\vz;\mat{\theta})$, where $\mat{\theta}$ is a learnable parameter and $\vz$ is a random variable drawn from a simple and fixed distribution. For example, a Gaussian mixture with $L$ equal-probability modes can be parameterized by $\tvc = \sum_{\ell=1}^L (\mat{\mu}_\ell + \mat{\Sigma}_\ell \mat{\epsilon}_l)\beta_l$, where $\mat{\mu_\ell}$ and $\mat{\Sigma}_\ell$ are two learnable parameters representing the mean and covariance of the $\ell$th mode respectively, and $\mat{\epsilon} \sim\mathcal{N}(0,\mI)$; $\beta_\ell$ is a random indicator variable where only one $\beta_\ell$ for $\ell=1, 2,\ldots, L$ is 1 and the others are $0$. In practice, $f(\vz;\theta)$ is implemented by a multilayer neural network. Our analysis can be naturally extended  to analyzing this model as long as the dimensions of $\tvc$ and $\mat{\theta}$ keep finite when the data dimension $n$ goes to infinity. More challenging situations, where the dimension of  $\mat{\theta}$ is proportional to $n$, will be explored in future works. 

Although our analysis is carried out in the asymptotic setting, numerical experiments show that our theoretical predictions can accurately capture the actual performance of the training algorithm at moderate dimensions. Our analysis also reveals several different phases of the training process that highly depend on the choice of the learning rates and noise strength. 
The analysis reveals a condition on the learning rates and the strength of noise  to have  successful training. Violating this condition results either oscillation or mode collapsing.
Despite its simplicity, the proposed model of GAN provides a new perspective and some insights for the study of more realistic models and more involved training algorithms.

%

\paragraph*{Acknowledgments}
This work was supported by the US Army Research Office under contract W911NF-16-1-0265 and by the US National Science Foundation under grants CCF-1319140, CCF-1718698, and CCF-1910410.
\newpage
{
\small
\bibliographystyle{IEEEtran}

}
\newpage
\renewcommand{\theequation}{S-\arabic{equation}}
\renewcommand \thesection {S-\Roman{section}}

%
%
\setcounter{equation}{0}
\setcounter{section}{0}
\part*{Supplementary Materials}
These Supplementary Materials provide additional information, detailed derivations and
proof of the results shown in the main text. Specifically, in Section~\ref{sec:phase} we provide 
a local stability analysis and draw the phase diagram in the case $d=1$ and $d=2$.
In Section
\ref{sec:dev-micro}, we present a heuristic derivation of the stochastic
differential equation (SDE) for the microscopic states. Next, in Section
\ref{sec:dev-macro-weak}, we show a derivation of the ODE for the
macroscopic states from the weak formulation of the PDE. We then establish
the full proof of the Theorem 1 in Section \ref{sec:proof-thm1}.
Finally, we present the local stability analysis of the ODE's fixed
points in Section \ref{sec:stability}.

\emph{Notation}: Throughout the paper, we use $\mI_d$ to denote the $d \times d$ identity matrix. Depending on the context, $\norm{\cdot}$ denotes either the $\ell_2$ norm of a vector or the spectral norm of a matrix. For any $x \in \R$, the floor operation $\lfloor x \rfloor$ gives the largest integer that is smaller than or equal to $x$. We denote $[\vv]_i$ the $i$th element of the vector $\vv$ and denote $[\mM]_{i,j}$ the element at $i$th row and $j$th column of the matrix $\mM$. Finally, $C(T)$ denotes   a constant that depends on the terminal time $T$, and $C$ denotes a general constant that does not depends on $T$ and $n$. Both $C$ and $C(T)$ can vary line to line.

\section{Phase diagram for the case $d=1 $ and $d=2$}\label{sec:phase}

In what follows, we provide a thorough study of all the fixed points of the ODE \eqref{eq:PCA-ODE} when the number of feature $d=1$ and $d=2$.   In particular, three major phases are identified under different settings of the learning rates $\tau$ and $\ttau$ with the fixed model parameters $\eta_{\text{T}}$, $\eta_{\text{G}}$, $\mat\Lambda$, and $\t{ \mat\Lambda}$ .

\paragraph{Phase diagram for $d=1$.}
By analyzing the local stabilities of these fixed points as illustrated in Figure \ref{fig:phase}(a), we obtain the phase diagram as shown in Figure \ref{fig:phase}(b).  For simplicity, we only present the result when $\eta_{\text{T}}=\eta_{\text{G}}=1$, and $ \mat\Lambda= \t{ \mat\Lambda}$, which is denoted by $\Lambda$ used in the remaining part of this section. Detailed derivations are presented in \ref{sec:stability}.

Even in this simplest case, we find there are  in total 5 types of fixed points, the locations of which are visualized
in the 3-dimensional space $( P, q, r)$ shown in Figure \ref{fig:phase}(a).
Each type of the fixed points has an intuitive meaning in terms of the two-player game between
$\calG$ and $\calD$.
We list the detailed information in Table \ref{tab:fixed-points},
in which we  define  a function 
$ 
\beta(\tau) = \begin{cases}
[1+ (\tfrac{\Lambda}{2}-\tfrac{\Lambda}{\tau})^{-1}]^{-1}, &\text{if } \tau\leq \frac{2\Lambda}{\Lambda+2}\\
+\infty, &\text{otherwise}
\end{cases}
$.

\emph{Noninformative phase:} We say that the ODE \eqref{eq:PCA-ODE} is in a noninformative phase if either a type-1 or type-2 fixed point in Table \ref{tab:fixed-points} is stable. In this case,   $P=0$, which indicates that the generator's parameter vector $\mV$  has no correlation with the true feature vector $\mU$. 
 In Figure \ref{fig:phase}(b), the region labeled as noninfo-1 is the stable region for the type-1 fixed point, and noninfo-2 is 
the stable region for the type-2 fixed point. The two regions have no overlap. However, we note that 
in noninfo-1, the type-3 fixed points can also be stable, in which case the stationary point of the ODE is determined by the initial condition.

\emph{Informative phase:} We say that the ODE \eqref{eq:PCA-ODE} is in an informative phase if neither  type-1 nor type-2 fixed point is stable, and if at least one fixed point of type-3 and type-5 is stable. In this case, it is guaranteed that $P$ is nonzero, indicating that the generator can achieve non-vanishing correlation with the real feature vector. In addition, the stable regions for the type-3 and type-5 fixed points are disjoint. They are shown in Figure \ref{fig:phase}(b) as info-1 and info-2, respectively. The difference between the two region is that, in info-1,   $q$ is exactly $0$ indicating that the discriminator is completely fooled, whereas  in info-2, $q$ is nonzero. 

\emph{Oscillating phase:} We say that the ODE \eqref{eq:PCA-ODE} is in an oscillating phase if none of  the fixed points  in Table \ref{tab:fixed-points} is  stable.
In this phase, limiting cycles emerge and the system will oscillate on these cycles indefinitely.
Moreover, we found two types of limiting cycles. 

To further illustrate the phase transitions, we draw ODE trajectories and phase portraits  
in Figure \ref{fig:d1} corresponding to different choices of the step sizes (from left to right, $\ttau = 0.03, 0.2, 0.4, 0.47)$.  

The two figures in the first column of Figure \ref{fig:d1} show a case in the Info-1 phase. The bottom red dot in Figure \ref{fig:phase}.(b) represents this configuration of the step sizes, where $\ttau/\tau$ is small. The top figure of Figure \ref{fig:d1}.(a) shows the dynamics of $P_t$, $q_t$ and $r_t$, and the bottom figure shows the phase portrait on $P-q$ plane. Top figure of Figure \ref{fig:d1}.(a) shows an interesting phenomenon that dynamics are  separated into two stages. At the first stage, $q_t$ (red dots, cosine similarity between  the true feature vector and discriminator's estimation) increases drastically from 0 to some value near 1, while $P_t$ (blue dots, cosine similarity between the true feature vector and generator's estimation) almost doesn't change.  Intuitively, at this stage, the discriminator learns the true model while the generator is unchanged. In the second stage, the generator start to fool the discriminator, where $|P_t|$ increases and $q_t$ decreases. In fact, these two-stage dynamics can be understood from the ODE \eqref{eq:PCA-ODE}: When $\tau/\tau$ is small, the process can be decomposed into two processes in different time scales. In particular, the discriminator is associated with the faster dynamics as $\tau \gg \ttau$, and the generator governs the slower dynamics. Figure 1 in the main text shows that this picture is still hold for multi-feature cases in the hierarchical dynamics.

The figures in the middle two columns of Figure \ref{fig:d1} show the two types of limiting cycles that can emerge in the oscillating phase. The middle two red dots in Figure \ref{fig:phase}.(b) represents these configurations of the step sizes. The last column of Figure \ref{fig:d1} shows another stable phase in Info-2. In this phase, $\tau/\tau$ is relatively large. The two time-scale dynamics are mixed, and another type of stable fixed points emerges. 
\vspace{-0.5em}
\paragraph{Phase diagram for $d=2$.} Figure \ref{fig:phased2} shows the phase diagram when $d=2$.  In particular, the two red lines between Info-1 and Noninfo-1 in Figure \ref{fig:phased2} are determined by the left inequality in \eqref{eq:con}. In Info-1, both feature vectors are recovered by the generator. The dynamics of this phase are shown in Figure 1.(a) in the main text. In the Half-info phase, only the feature vector with the larger signal-to-noise ratio is recovered. The dynamics of this phase are shown in Figure 1.(c) in the main text. The blue line between Info-1 and oscillating phases shows the boundary between oscillation state and stable state.  


\begin{table*}[tbp]
\caption{\label{tab:fixed-points}List of the fixed points of the ODE
\eqref{eq:PCA-ODE} when $d=1$ and ${\Lambda}=\t{{\Lambda}}$.}
\centering%
\addtolength{\leftskip} {-2cm}
\addtolength{\rightskip}{-2cm}
\begin{tabular}{>{\centering}m{0.4cm}>{\centering}m{2cm}>{\centering}m{2.7cm}>{\centering}m{3.5cm}>{\centering}m{3.2cm}}
\hline 
Type & Location & Existence& Stable Region & Intuitive Interpretation\tabularnewline
\hline 
\hline 
1 & $ P=q=0$, $r=0$ & always &  
$  \tau > \Lambda^{2}$,  $\tfrac{\ttau}{\tau}<\tfrac{\tau+\Lambda}{\Lambda}$ 
& Both $\calG$ and $\calD$ fail, and they are uncorrelated\tabularnewline
\hline 
2 & $ P=q=0$ $r=\pm r^\ast\neq0$ 
& $\tfrac{\ttau}{\tau}\geq\tfrac{\tau+\Lambda}{\Lambda}$ or $\tfrac{\ttau}{\tau}\leq 1-\tfrac{\tau}{2}$
& 
$\max\{ 2,\tfrac{\tau+\Lambda}{\Lambda}\} \leq \tfrac{\ttau}{\tau}\leq \beta(\tau)$ 
& Both $\calG$ and $\calD$ fail, and they are correlated\tabularnewline
\hline 
3 & $q=r=0$ $ \abs{P} \in (0,1]$ & always 
& 
$\abs{P}=1$ is stable if 
$\tfrac{\tilde{\tau}}{\tau}\leq\min\{\tfrac{2\tau}{\Lambda},\max\{\tfrac{\tau^{2}\Lambda^{-1}}{\abs{\tau-\Lambda}},4\}\}  $
& $\calG$ wins and $\calD$ loses 
\tabularnewline
\hline 
4 & $ P=r=0$ $q = \pm q^\ast\neq 0$ 
& always
& always unstable & $\calG$ loses and $\calD$ wins\tabularnewline
\hline 
5 & None of $P$, $q$ or $r$ is zero 
&  not always, at most 8 fixed points
&  can be computed numerically
& Both $\calG$ and $\calD$ are informative
\tabularnewline
\hline 
\end{tabular}
\end{table*}

\begin{figure*}[ht!]
\center
\subfloat[]{ \includegraphics[scale=0.5]{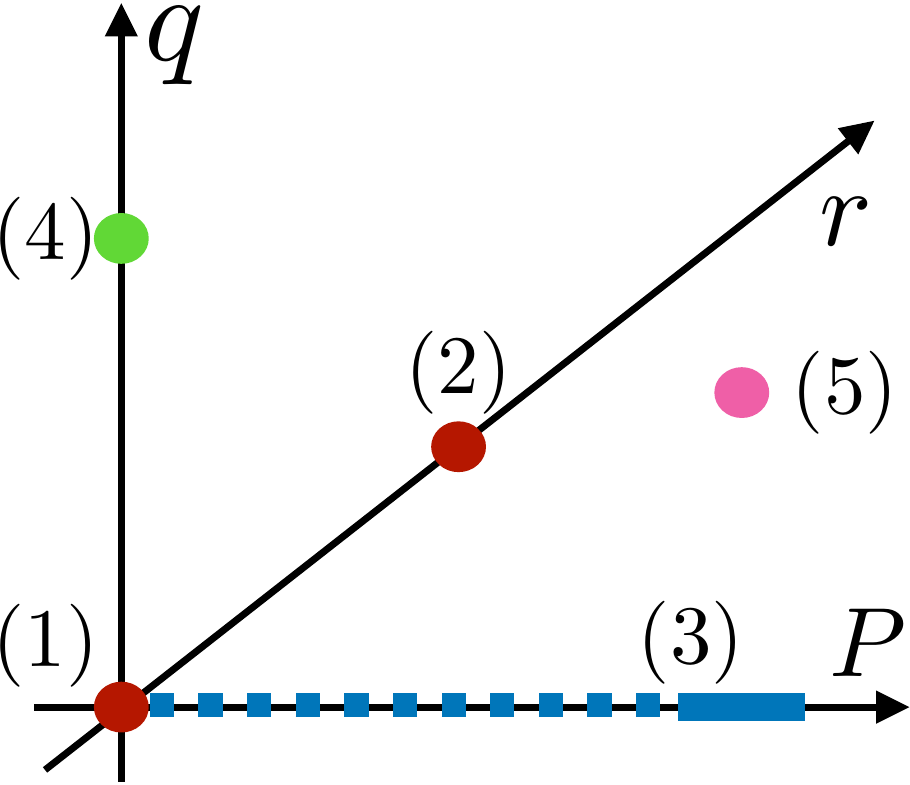} }
\hspace{3em}
\subfloat[] {  \includegraphics[scale=.85]{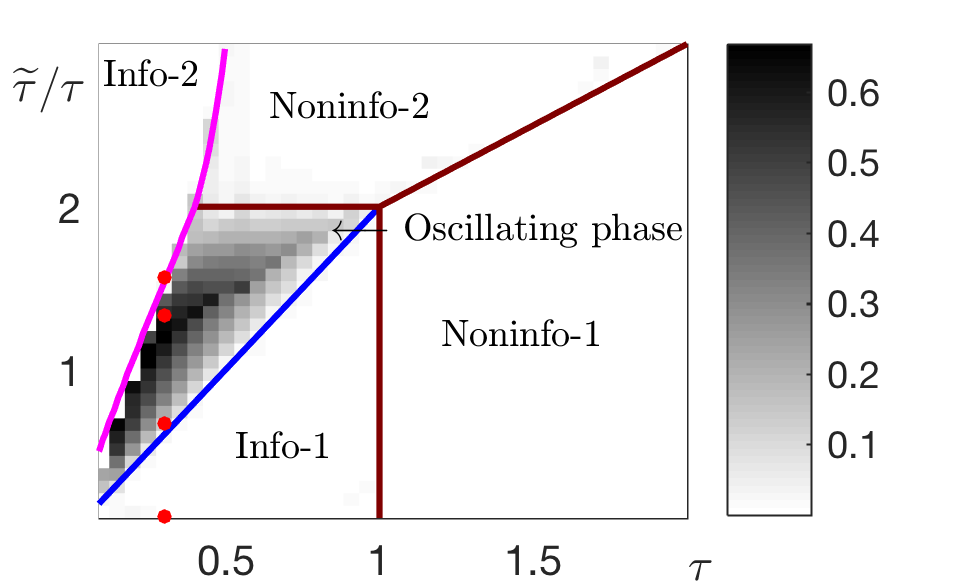}}
\caption{\label{fig:phase}
(a): The locations of the five types of fixed points of the ODE \eqref{eq:PCA-ODE}. Their properties are listed in Table 1. 
(b): The phase diagram for the stationary state of the ODE \eqref{eq:PCA-ODE}. The colored lines illustrate the theoretical prediction of the boundaries between the different phases. Simulations results for a single numerical experiment are also shown to illustrate the oscillating phase: Each grey square represents the value of
$\frac{1}{200}\int_{800}^{1000} [(P_t - \langle P_t \rangle)^2 + (q_t - \langle q_t \rangle)^2 + (r_t - \langle r_t \rangle)^2]\dif t$ where $ \langle P_t \rangle=\frac{1}{200}\int_{800}^{1000} P_t \dif t$, and  
$\langle q_t \rangle$ and $ \langle r_t \rangle$ are defined similarly. Note that the above quantity measures the variation (over time) of the training process as it approaches steady states. 
We see that the variation is indeed nonzero in the oscillating phase (see Figure \ref{fig:d1}), whereas the variation is close to zero in all other phases. 
}
\vspace{1em}
\end{figure*}

\begin{figure*}[ht!]
\center
{ \includegraphics[scale=0.78]{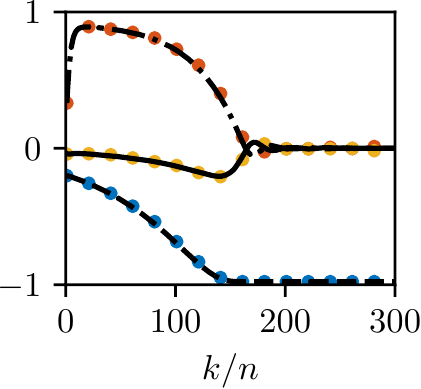}   }
{ \includegraphics[scale=0.78]{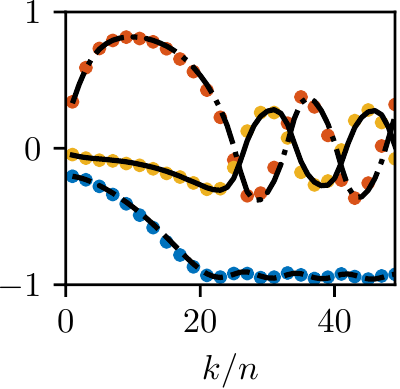}  }
{ \includegraphics[scale=0.78]{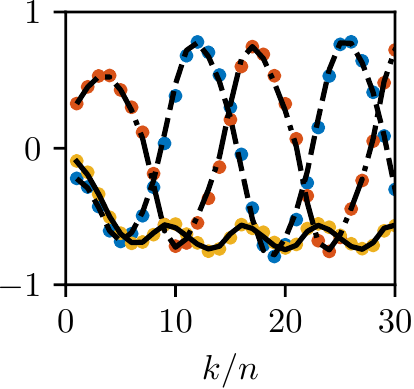}  }
{ \includegraphics[scale=0.78]{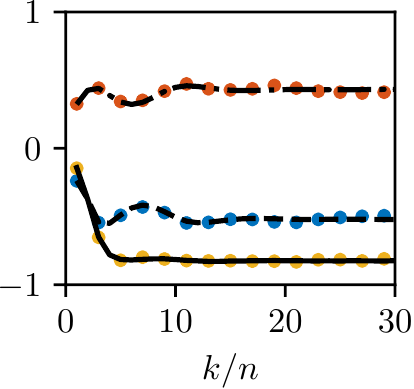}}
\\
\subfloat[] { \includegraphics[scale=0.93]{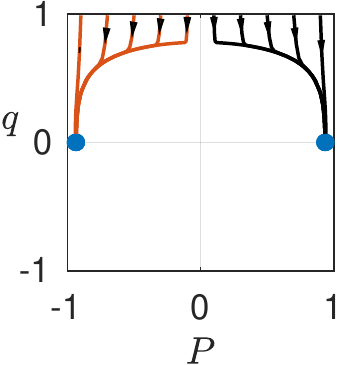} \hspace{.5em}}
\subfloat[] { \includegraphics[scale=0.93]{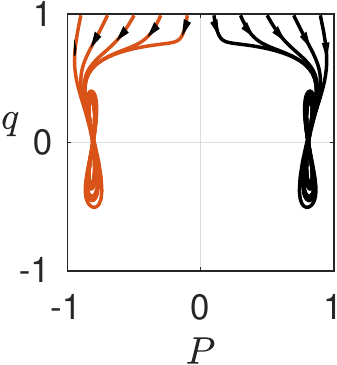} \hspace{.5em}}
\subfloat[] { \includegraphics[scale=0.93]{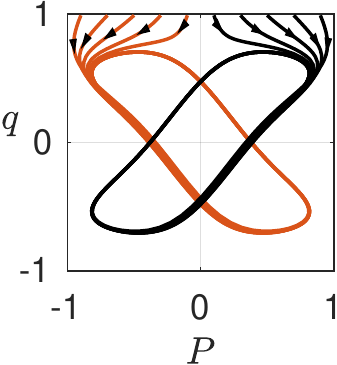} \hspace{.5em}}
\subfloat[] { \includegraphics[scale=0.93]{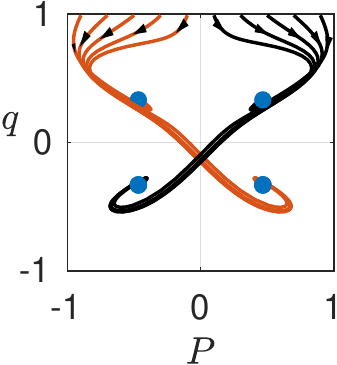} } 
\caption{\label{fig:d1}
Macroscopic dynamics of Example \ref{ex:PCA} with $d=1$.  In the first row, the red, blue and yellow dots represent $ P_t$, $ q_t$, and $ r_t$ respectively of the experimental results of a single trial. 
The black curves under the dots are theoretical predictions given by the ODE \eqref{eq:PCA-ODE}.
We set a fix the discriminator's learning rate $ \tau=0.3$ and vary the generator's learning rate $  \ttau=0.03,\;0.2,\;0.4,\;0.47$ from left to right column. These parameter settings are marked by the four red dots in the phase diagram in Figure~\ref{fig:phase}.
The second row is the phase portraits of the trajectories shown in the first row onto the $P$--$q$ plane.
Figure (a) shows a case in the phase of info-1, where a subset of type (3) fixed points  are stable. Figure (b) and (c) are in the oscillating phase, and (d) is in info-2, where the fixed points of type-5 are stable.
The blue dots in the figures show the stable fixed points.
}
\end{figure*}

\begin{figure*}[ht!]
\center{\includegraphics[scale=1]{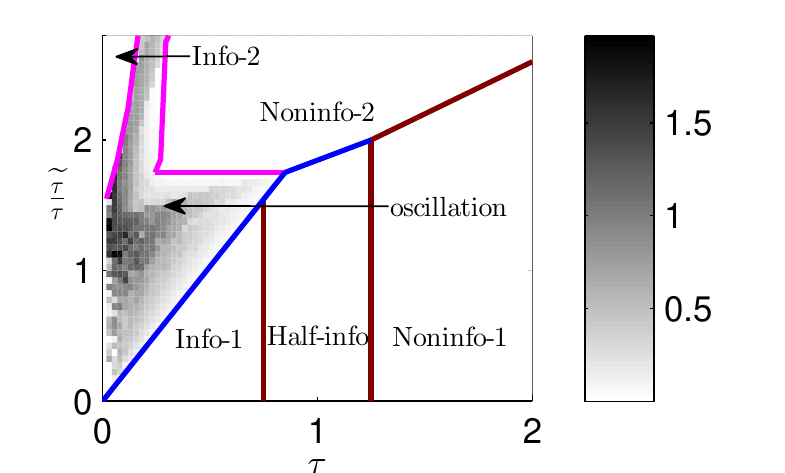}}

\caption{\label{fig:phased2}The phase diagram for the stationary states of the ODE \eqref{eq:PCA-ODE} when $d=2$. This phase diagram is generated by numerically computing the fixed points and eigenvalues of the Jacobian of the ODE \eqref{eq:PCA-ODE}.} 
\end{figure*}

\section{Heuristic derivations of the dynamics of the microscopic states\label{sec:dev-micro}}

In this section, we derive the stochastic differential equations \eqref{eq:SDE}
in the main text for the microscopic states in a non-rigorous way. 
Specifically, we directly discard higher-order terms without any justification, in order to highlight the main ideas. In Section \ref{sec:proof-thm1}, we rigorously justify these steps by providing bounds on those terms.

Our starting point is the iterative algorithm (\ref{eq:sgd}) in the
main text. Substituting the objective function $\mathcal{L}$ defined
in (\ref{eq:def-obj}) into (\ref{eq:sgd}), we have

\begin{align}
\vw_{k+1} & =\vw_{k}+\tfrac{\tau}{n}\big[\vy_{k}f(\vy_{k}^{\T}\vw_{k})-\tvy_{2k}\t f(\tvy_{2k}^{\T}\vw_{k})-\lambda\vw_{k}H^{\prime}(\vw_{k}^{\T}\vw_{k})\big]\label{eq:update-w}\\
\mV_{k+1} & =\mV_{k}+\tfrac{\ttau}{n}\big[\vw_{k}\tvc_{2k+1}^{\T}\tf(\tvy_{2k+1}^{\T}\vw_{k})-\lambda\mV_{k}\opdiag(H^{\prime}(\mV_{k}^{\T}\mV_{k}))\big],\label{eq:update-V}
\end{align}
where $\vy_{k}$ and $\tvy_{k}$ are true and fake samples generated
according to (\ref{eq:real-G}) and (\ref{eq:gen-model}) respectively.
The two functions $f$, $\tf$ stand for $f(x)= \frac{\dif}{\dif x}F( \h D(x))$
and $\tf(x)= \frac{\dif}{\dif x} \t F( \h D(x))$. The function $H^\prime$ is derivative
of $H$. If the input of $H^{\prime}(\cdot)$ is a matrix, $H^{\prime}$
applies to the input matrix element-wisely. The operation $\opdiag(\mA)$
is a diagonal matrix of $\mA$, where the off-diagonal term are set
to zero.

We note that the elements of $\vw_{k}$ and $\mV_{k}$ are $\mathcal{O}(\frac{1}{\sqrt{n}})$
number as the norm of $\vw_{k}$ and the norms of column vectors of
$\mV_{k}$ are all $\mathcal{O}(1)$ numbers. To investigate the dynamics
of the microscopic state, it is convenient to rescale $\vw_{k}$ and
$\mV_{k}$ by a factor of $\sqrt{n}$. We define $\hvu_{i}$ and $\hvv_{k,i}$
as the column view of the $i$'th row of the matrices $\sqrt{n}\mU$
and $\sqrt{n}\mV_{k}$ respectively, and $\hw_{k,i}\bydef\sqrt{n}[\vw_{k+1}]_{i}$.
The update rule of $((\text{\ensuremath{\hvu}}_{i},\hvv_{k,i},\hw_{k,i})_{i=1,\ldots,n})_{k=0,1,2,\ldots}$
is
\begin{align}
\hw_{k+1,i}-\hw_{k,i}= & \tfrac{\tau}{n}\bigg[\big(\hvu_{i}^{\T}\vc_{k}+\sqrt{n\eta_{\text{T}}}a_{k,i}\big)f_{k}-\left(\hvv_{k,i}^{\T}\tvc_{2k}+\sqrt{n\eta_{\text{G}}}\t a_{2k,i}\right)\tf_{2k}-\lambda H^{\prime}(z_k)\hw_{k,i}\bigg],\label{eq:update-hw}\\
\hvv_{k+1,i}-\hvv_{k,i}= & \tfrac{\ttau}{n}\bigg[\hw_{k,i}\tvc_{2k+1}\tf_{2k+1}-\lambda\opdiag(H^{\prime}(\mS_k))\hvv_{k,i}\bigg],\label{eq:update-hvv}
\end{align}
where $a_{k,i}$, $\t a_{k,i}$ are the $i$th elements of $\va_k$ and $\tva_k$ respectively, and $f_{k}$ and $\tf_{k}$ are shorthands for 
\begin{align*}
f_{k} & =f(\vy_{k}^{\T}\vw_{k}/\sqrt{n})=f\Big(\vq_{k}^{\T}\vc_{k}+\sqrt{\tfrac{\eta_{\text{T}}}{n}}{\textstyle \sum_{j=1}^{n}}a_{k,j}\hw_{k,j}\Big)\\
\t f_{k} & =\tf(\tvy_{k}^{\T}\vw_{\left\lfloor k/2\right\rfloor }/\sqrt{n})=\tf\Big(\vr_{\left\lfloor k/2\right\rfloor }^{\T}\tvc_{k}+\sqrt{\tfrac{\eta_{\text{G}}}{n}}{\textstyle \sum_{j=1}^{n}}\t a_{k,j}\hw_{\left\lfloor k/2\right\rfloor ,j}\Big),
\end{align*}
respectively, and 
the empirical macroscopic quantities $\vq_k$, $\vr_k$, $\vz_k$ and $\mS_k$ are defined as follows 
\begin{equation}
\begin{aligned} \label{eq:def-qrzS}
\vq_{k} & \bydef\mU^{\T}\vw_{k}=\tfrac{1}{n}{\textstyle \sum_{i=1}^{n}}\hvu_{i}\hw_{i}, & \qquad & \vr_{k}\bydef\mV_{k}^{\T}\vw_{k}=\tfrac{1}{n}{\textstyle \sum_{i=1}^{n}}\hvv_{k,i}\hw_{i},\\
z_{k} & \bydef\vw_{k}^{\T}\vw_{k}=\tfrac{1}{n}{\textstyle \sum_{i=1}^{n}}\hw_{k,i}^{2}, &  & \mS_{k}\bydef\mV_{k}^{\T}\mV_{k}=\tfrac{1}{n}{\textstyle \sum_{i=1}^{n}}\hvv_{k,i}\hvv_{k,i}^{\T},\\
\mP_k&\bydef \mU^\T \mV_k = \tfrac{1}{n}{\textstyle \sum_{i=1}^{n}}{\hvu_i \hvv_{k,i}^\T}.
\end{aligned}
\end{equation}
The matrix $\mP_k$ is not used in this section, but we put it here with the other macroscopic quantities for future reference.

Now we derive \eqref{eq:SDE} from \eqref{eq:update-hw} and \eqref{eq:update-hvv}.

First, it is trivial to get the first equation of the SDE $\dif\hvu_{t}=0$
in (\ref{eq:SDE}) in the main text, since $\hvu_{i}$ does not change
over time.

Next, we derive the second equation in (\ref{eq:SDE}). Averaging
over $\tvc_{2k+1}$ and $\tva_{2k+1}$ on the both sides of (\ref{eq:update-hvv}),
we get

\[
\begin{aligned}
&\left\langle \hvv_{k+1,i}-\hvv_{k,i}\right\rangle _{\tvc_{2k+1},\tva_{2k+1}}
\\
&=\tfrac{\ttau}{n}\bigg[\left\langle \tf\Big(\vr_{k}^{\T}\tvc+\sqrt{\tfrac{\eta_{\text{G}}}{n}}{\textstyle \sum_{j=1}^{n}}[\tva]_{j}\hw_{k,j}\Big]\Big)\tvc\right\rangle _{\tvc,\tva}\hw_{k,i}-\lambda\opdiag(H^{\prime}(\mS_k))\hvv_{k,i}\bigg].
\end{aligned}
\]
The bracket $\langle \cdot \rangle_{\tvc,\tva}$ here denotes the average over $\tvc
\sim\mathcal{P}_{\tvc}$, 
and standard Gaussian vector $\tva$, where $\tvc$ and $\tva$ are the random variables generating the fake sample in the generator as described in \eqref{eq:gen-model}. 
Noting that $\tva$ is a Gaussian vector, the term $\tfrac{1}{\sqrt{n}}\sum_{j=1}^{n}[\tva]_{j}\hw_{k,j}$
in the above equation is also a Gaussian random variable, whose mean
is zero and variance is $z_{k}$, which is defined in \eqref{eq:def-qrzS}. Therefore, we have 
\begin{equation}
\left\langle \hvv_{k+1,i}-\hvv_{k,i}\right\rangle _{\tvc_{2k+1},\tva_{2k+1}}=\tfrac{\ttau}{n}\Big[\t{\vg}_{k}\hw_{k,i}+\mL_{k}\hvv_{k,i}\Big],\label{eq:m-hvv}
\end{equation}
where 
\begin{align}
\t{\vg}_{k} & =\left\langle \tf\Big(\vr_{k}^{\T}\tvc+\sqrt{z_{k}\eta_{\text{G}}}e\Big]\Big)\tvc\right\rangle _{\tvc,e}\label{eq:def-tvgk}\\
\mL_{k} & =-\lambda\opdiag(H^{\prime}(\mS_{k})),\label{eq:def-Lk}
\end{align}
where $\left\langle \cdot\right\rangle _{\t{\vc},e}$ denotes the
average over $\tvc\sim\mathcal{P}_{\tvc}$ and $e\sim\mathcal{N}(0,1)$.
In addition, from (\ref{eq:update-hvv}), we also know that the second
moment 
\begin{equation}
\left\langle \left(\hvv_{k+1,i}-\hvv_{k,i}\right)^{2}\right\rangle _{\tvc_{2k+1},\tva_{2k+1}}=\mathcal{O}(n^{-\frac{3}{2}}).\label{eq:m-hvv2}
\end{equation}
The moments estimations (\ref{eq:m-hvv}) and (\ref{eq:m-hvv2}) imply
the second equation in (\ref{eq:SDE}) in the main text. Since the
second moments growth smaller than $\mathcal{O}(n^{-1})$, the differential
equation for $\hvv_{t}$ has no diffusion term.

Finally, we derive the last equation in (\ref{eq:SDE}) in the main
text from the update rule of $\hw_{k}$ (\ref{eq:update-hw}). We
observe that both the terms inside the function $f$ and outside of
$f$ in (\ref{eq:update-hw}) depend on $a_{k,i}$. Using Taylor's
expansion, we linearize the contribution of $a_{k,i}$ to the function
$f$: 
\begin{align}
f_{k} & =f\Big(\vq_{k}^{\T}\vc_{k}+\sqrt{\tfrac{\eta_{\text{T}}}{n}}{\textstyle \sum_{j\neq i}}a_{k,j}\hw_{k,j}+\sqrt{\tfrac{\eta_{\text{T}}}{n}}a_{k,i}\hw_{k,i}\Big)\nonumber \\
 & =f(\vq_{k}^{\T}\vc_{k}+\sqrt{\tfrac{\eta_{\text{T}}}{n}}{\textstyle \sum_{j\neq i}}a_{k,j}\hw_{k,j})+f^{\prime}(\vq_{k}^{\T}\vc_{k}+\sqrt{\tfrac{\eta_{\text{T}}}{n}}{\textstyle \sum_{j\neq i}}a_{k,j}\hw_{k,j})\sqrt{\tfrac{\eta_{\text{T}}}{n}}a_{k,i}\hw_{k,i}+\mathcal{O}(\tfrac{1}{n}).\label{eq:fk-exp}
\end{align}
Similarly, we have 
\begin{align}
\tf_{2k} & =\tf(\vr_{k}^{\T}\tvc_{2k}+\sqrt{\tfrac{\eta_{\text{G}}}{n}}{\textstyle \sum_{j\neq i}}\t a_{k,j}\hw_{k,j}+\sqrt{\tfrac{\eta_{\text{G}}}{n}}\t a_{2k,i}\hw_{k,i})\nonumber \\
 & =\tf(\vr_{k}^{\T}\tvc_{2k}+\sqrt{\tfrac{\eta_{\text{G}}}{n}}{\textstyle \sum_{j\neq i}}\t a_{k,j}\hw_{k,j})+\tf^{\prime}(\vr_{k}^{\T}\tvc_{2k}+\sqrt{\tfrac{\eta_{\text{G}}}{n}}{\textstyle \sum_{j\neq i}}\t a_{k,j}\hw_{k,j})\sqrt{\tfrac{\eta_{\text{G}}}{n}}\t a_{2k,i}\h w_{k,i}+\mathcal{O}(\tfrac{1}{n})\label{eq:tfk-exp}
\end{align}
Substituting (\ref{eq:fk-exp}) and (\ref{eq:tfk-exp}) into (\ref{eq:update-hw}),
we have

\begin{equation}
\begin{aligned}
&\frac{\hw_{k+1,i}-\hw_{k,i}}{\tau/n}
\\
&=  \hvu_{i}^{\T}\vc_{k}f(\vq_{k}^{\T}\vc_{k}+\sqrt{\tfrac{\eta_{\text{T}}}{n}}{\textstyle \sum_{j\neq i}}a_{k,j}\hw_{k,j})-\hvv_{k,i}^{\T}\tvc_{2k}\tf(\vr_{k}^{\T}\tvc_{2k}+\sqrt{\tfrac{\eta_{\text{G}}}{n}}{\textstyle \sum_{j\neq i}}\t a_{k,j}\hw_{k,j})\\
 & +\hw_{k,i}\Big[a_{k,i}^{2}f^{\prime}(\vq_{k}^{\T}\vc_{k}+\sqrt{\tfrac{\eta_{\text{T}}}{n}}{\textstyle \sum_{j\neq i}}a_{k,j}\hw_{k,j})-\t a_{k,i}^{2}\tf^{\prime}(\vr_{k}^{\T}\tvc_{2k}+\sqrt{\tfrac{\eta_{\text{G}}}{n}}{\textstyle \sum_{j\neq i}}\t a_{2k,j}\hw_{k,j})-\lambda H^{\prime}(z_{k})\Big]\\
 & +\sqrt{n}\Big[a_{k,i}f(\vq_{k}^{\T}\vc_{k}+\sqrt{\tfrac{\eta_{\text{T}}}{n}}{\textstyle \sum_{j\neq i}}a_{k,j}\hw_{k,j})+\t a_{2k,i}\tf^{\prime}(\vr_{k}^{\T}\tvc_{2k}+\sqrt{\tfrac{\eta_{\text{G}}}{n}}{\textstyle \sum_{j\neq i}}\t a_{2k,j}\hw_{k,j})\Big]+\delta_{k,i},
\end{aligned}
\label{eq:hw-exp}
\end{equation}
where $\delta_{k,i}$ collects all higher-order terms whose contributions
will vanish as $n\to\infty$.
From this equation, we can already infer the SDE (\ref{eq:SDE}). Specifically, on the right hand side of $\eqref{eq:hw-exp}$,
the terms in the first two lines correspond to the drift term in the
SDE. Furthermore, the first term in the third line in $\eqref{eq:hw-exp}$ contributes
to the SDE as a Brownian
motion. More precisely, we can derive the third equation of the SDE
(\ref{eq:SDE}) in the main text by the moments estimations. Specifically,
the first-order moment is
\begin{align}
\left\langle \hw_{k+1,i}-\hw_{k,i}\right\rangle _{\vc_{k,}\va_{k},\tvc_{2k},\tva_{2k}}=\tfrac{\tau}{n}\Big[ & \hvu_{i}^{\T}\vg_{k}-\hvv_{k,i}^{\T}\t{\vg}_{k}+\hw_{k,i}h_{k}\Big]+\mathcal{O}(n^{-\frac{3}{2}})\label{eq:hw-mean}
\end{align}
where $\t{\vg}_{k}$ is defined in (\ref{eq:def-tvgk}), and 
\begin{align}
\vg_{k} & =\left\langle \vc f(\vq_{k}^{\T}\vc+\sqrt{z_{k}\eta_{\text{T}}}e)\right\rangle _{\vc,e}\label{eq:def-vgk}\\
h_{k} & =\eta_{\text{T}}\left\langle f^{\prime}(\vq_{k}^{\T}\vc+\sqrt{z_{k}\eta_{\text{T}}}e)\right\rangle _{\vc,e}-\t{\eta}_{\text{G}}\left\langle \tf^{\prime}(\vr_{k}^{\T}\t{\vc}+\sqrt{z_{k}\eta_{\text{G}}}e)\right\rangle _{\tvc,e}-\lambda H^{\prime}(z_{k}).\label{eq:def-hk}
\end{align}
The second moment is
\begin{equation}
\left\langle \left(\hw_{k+1,i}-\hw_{k,i}\right)^{2}\right\rangle _{\vc_{k,}\va_{k},\tvc_{2k},\tva_{2k}}=\tfrac{\tau^{2}}{n}b_{k}+\mathcal{O}(n^{-\frac{3}{2}}),\label{eq:hw-var}
\end{equation}
where 
\begin{equation}
b_{k}=\eta_{\text{T}}\left\langle f^{2}(\vq_{k}^{\T}\vc+\sqrt{z_{k}\eta_{\text{T}}}e)\right\rangle _{\vc,e}+\eta_{\text{G}}\left\langle \tf^{2}(\vr_{k}^{\T}\t{\vc}+\sqrt{z_{k}\eta_{\text{G}}}e)\right\rangle _{\tvc,e}.\label{eq:def-bk}
\end{equation}
From the (\ref{eq:hw-mean}) and (\ref{eq:hw-var}), we derive the
SDE for $\hw_{t}$ in (\ref{eq:SDE}) in the main text.

\section{Derive the ODE in Theorem 1 from the weak formulation of the PDE\label{sec:dev-macro-weak}}

In this section, we show how to derive the ODE \eqref{eq:ODE} from the weak formulation
of the PDE \eqref{eq:weak}. Choosing the test function $\varphi$ being each element
of $\hvu\hvv^{\T},\;\hvu\hw,\;\hvv\hw,\;\hvv\hvv^{\T},\;\hw^{2}$,
and substituting those $\varphi$ into the weak formulation of the
PDE (\ref{eq:weak}), we will get the
ODE  \eqref{eq:ODE}  as presented in Theorem~1. In what follows, we provide
additional details of this derivation.

We first derive the first ODE $\frac{\dif}{\dif t}\mP_{t}=\ldots$
in (\ref{eq:ODE}).
Let $\varphi=[\h{\vu}]_{\ell}[\h{\vv}]_{\ell'}$, $\text{\ensuremath{\ell}},\ell'=1,2,\ldots,d$,
we have $\nabla_{\hat{\vv}}\varphi=[\h{\vu}]_{\ell}\vs_{\ell'}$,
where $\vs_{\ell^{\prime}}$ is the $\ell^\prime$th canonical basis ({\em i.e.,} all elements in $\vs_{\ell^\prime}$ are zeros, except that $\ell^\prime$th element is 1). From the PDE
(\ref{eq:weak}) in the main text, we have $\forall\ell,\ell'=1,2,\ldots,d$:
\begin{align*}
\left\langle \mu_{t},\varphi(\h{\vu},\h{\vv},\hw)\right\rangle  & =\left\langle \mu_{t},[\h{\vu}]_{\ell}[\h{\vv}]_{\ell'}\right\rangle  =[\mP_{t}]_{\ell,\ell'},
\end{align*}
\begin{align*}
\left\langle \mu_{t},(\h w\t{\vg}_{t}^{\T}+\h{\vv}^{\T}\mL_{t})\nabla_{\hat{\vv}}\varphi\right\rangle  & =\left\langle \mu_{t},([\h{\vu}]_{\ell}\h w)[\t{\vg}_{t}]_{\ell'}+([\h{\vu}]_{\ell}\h{\vv}^{\T})[\mL_{t}]_{:,\ell'}\right\rangle \\
 & =[\vq_{t}]_{l}[\t{\vg}_{t}]_{\ell'}+[\mP_{t}]_{\ell,:}[\mL_{t}]_{:,\ell'},
\end{align*}
where $[\mP_{t}]_{\ell,:}$ and $[\mL_{t}]_{:,\ell'}$ are $\ell$th row of $\mP_t$ and $\ell^\prime$th column of $\mL$, respectively.
In addition, we know that $\frac{\partial}{\partial\h w}\varphi=\frac{\partial^{2}}{\partial\h w^{2}}\varphi=0$.
Combining above results, we can recover the first ODE 
in (\ref{eq:ODE}).

Next, we derive the second ODE $\frac{d\vq_{t}}{dt}=\ldots$ in (\ref{eq:ODE}). Let $\varphi=[\h{\vu}]_{\ell}\h w$,
$\ell=1,2,\ldots,d$. We have $\nabla_{\hat{\vv}}\varphi=0$, $\frac{\partial}{\partial\h w}\varphi=[\h{\vu}]_{\ell}$
and $\frac{\partial^{2}}{\partial\h w^{2}}\varphi=0$. Then $\forall\ell=1,2,\ldots,d$,
\begin{align*}
\left\langle \mu_{t},\varphi(\h{\vu},\h{\vv},\hw)\right\rangle  & =\left\langle \mu_{t},[\h{\vu}]_{\ell}\h w\right\rangle  =[\vq_{t}]_{\ell}
\end{align*}
and
\begin{align*}
\left\langle \mu_{t},(\h{\vu}^{\T}\vg_{t}-\h{\vv}^{\T}\t{\vg}_{t}+h_{t}\h w)\frac{\partial}{\partial\h w}\varphi\right\rangle  & =\left\langle \mu_{t},(\h{\vu}^{\T}\vg_{t}-\h{\vv}^{\T}\t{\vg}_{t}+h_{t}\h w)[\h{\vu}]_{\ell}\right\rangle \\
 & =[\vg_{t}]_{\ell}-[\mP_{t}]_{\ell}\t{\vg}_{t}+[\vq_{t}]_{\ell}h_{t}.
\end{align*}
With above results, we can obtain the second ODE in (\ref{eq:ODE}).

Next, let's derive the ODE for $\frac{d\mS_{t}}{dt}$. We set $\varphi=[\h{\vv}]_{\ell}[\h{\vv}]_{\ell'}$. If $\ell\neq\ell'$,
we have $\nabla_{\hat{\vv}}\varphi=[\h{\vv}]_{\ell}\vs_{\ell'}+[\h{\vv}]_{\ell'}\vs_{\ell}$,
where $\vs_{\ell^{\prime}}$ is the $\ell^\prime$th canonical basis.
Then 
\[
\left\langle \mu_{t},\varphi(\h{\vu},\h{\vv},\hw)\right\rangle =[\mS_{t}]_{\ell,\ell'}
\]
and
\begin{align*}
\left\langle \mu_{t},(\h w\t{\vg}_{t}^{\T}+\h{\vv}^{\T}\mL_{t})\nabla_{\hat{\vv}}\varphi\right\rangle  & =\left\langle \mu_{t},([\h{\vv}]_{\ell}\h w)[\t{\vg}_{t}]_{\ell'}+([\h{\vv}]_{\ell}\h{\vv}^{\T})[\mL_{t}]_{:,\ell'}\right\rangle \\
 & \quad+\left\langle \mu_{t},([\h{\vv}]_{\ell'}\h w)[\t{\vg}_{t}]_{\ell}+([\h{\vv}]_{\ell'}\h{\vv}^{\T})[\mL_{t}]_{:,\ell}\right\rangle \\
 & =[\vr_{t}]_{\ell}[\t{\vg}_{t}]_{\ell'}+[\t{\vg}_{t}]_{\ell}[\vr_{t}]_{\ell'}+[\mS_{t}]_{\ell,:}[\mL_{t}]_{:,\ell'}+[\mL_{t}]_{\ell,:}[\mS_{t}]_{:,\ell'}
\end{align*}
If $\ell=\ell'$, we have $\nabla_{\hat{\vv}}\varphi=2[\h{\vv}]_{\ell}\vs_{\ell}$,
then 
\[
\left\langle \mu_{t},\varphi(\h{\vu},\h{\vv},\hw)\right\rangle =[\mS_{t}]_{\ell,\ell}
\]
and
\[
\left\langle \mu_{t},(\h w\t{\vg}_{t}^{\T}+\h{\vv}^{\T}\mL_{t})\nabla_{\hat{\vv}}\varphi\right\rangle =2([\vr_{t}]_{\ell}[\t{\vg}_{t}]_{\ell}+[\mS_{t}]_{\ell,:}[\mL_{t}]_{:,\ell})
\]
Plugging back the above two equations and combining the fact that
$\frac{\partial}{\partial\h w}\varphi=\frac{\partial^{2}}{\partial\h w^{2}}\varphi=0$,
we recover the ODE of $\frac{d\mS_{t}}{dt}$.

The rest two ODEs can be obtained in the similar way by letting $\varphi$
to be each distinct component of $\h{\vv}\h w$ and $\h w^{2}$.

\section{Proof of Theorem 1\label{sec:proof-thm1}}

In this section, we prove Theorem 1 shown in the main text. In the previous
section, we have already provided a derivation of the ODE in Theorem 1 from
the weak formulation of the PDE for the microscopic states. In this
section, we follow a different path to prove the theorem without referencing
the PDE, because it is easier to establish the rigorous bound of the convergence
rate. Thus, the proof itself also provides another derivation of the
ODE, where the most relevant part is Lemma~\ref{lem:c2}.

\subsection{Sketch of the proof}

The proof follows the standard procedure of the convergence of stochastic
processes \cite{billingsley2013convergence,kushner2003stochastic}. We here build the whole proof on Lemma 2 in the
supplementary materials of \cite{Wang2018}. For reader's convenient, we present
that lemma below.
\begin{lem}[Lemma 2 in the supplementary materials  of \cite{Wang2018}] \label{lem:base}
 Consider a sequence
of stochastic process $\{\vx_{k}^{(n)},k=0,1,2,\ldots,\left\lfloor nT\right\rfloor \}_{n=1,2,\ldots}$,
with some constant $T>0$. If $\vx_{k}^{(n)}$ can be decomposed into
three parts
\begin{equation}
\vx_{k+1}^{(n)}-\vx_{k}^{(n)}=\tfrac{1}{n}\phi(\vx_{k}^{(n)})+\mat{\rho}_{k}^{(n)}+\mat{\delta}_{k}^{(n)}\label{eq:sp-x}
\end{equation}
such that

(C.1) The process $\sum_{k^{\prime}=0}^{k}\mat{\rho}_{k^{\prime}}^{(n)}$
is a martingale, and $\EE\norm{\mat{\rho}_{k}^{(n)}}^{2}\leq C(T)/n^{1+\epsilon_{1}}$
for some positive $\e_{1}$;

(C.2) $\EE\norm{\mat{\delta}_{k}^{(n)}}\leq C(T)/n^{1+\e_{2}}$ for
some positive $\e_{2};$

(C.3) $\phi(\vx)$ is a Lipschitz function, i.e., $\norm{\phi(\vx)-\phi(\t{\vx})}\leq C\norm{\vx-\t{\vx}}$;

(C.4) $\EE\norm{\vx_{k}^{(n)}}^{2}\leq C$ for all $k\leq\left\lfloor nT\right\rfloor $;

(C.5) $\EE\norm{\vx_{0}^{(n)}-\vx_{0}^{\ast}}\leq C/n^{\e_{3}}$ for
some positive $\e_{3}$ and a deterministic vector $\vx_{0}^{\ast}$,

then we have 
\[
\norm{\vx_{k}^{(n)}-\vx(\tfrac{k}{n})}\leq C(T)n^{-\min\{\frac{1}{2}\e_{1},\e_{2},\e_{3}\}},
\]
where $\vx(t)$ is the solution of the ODE
\[
\tfrac{\dif}{\dif t}\vx(t)=\phi(\vx(t)),\quad\text{ with }\vx(0)=\vx_{0}^{\ast}.
\]
\end{lem}
In Theorem 1, the stochastic process is the macroscopic states $\{\mM_{k},k=0,1,\ldots\}$,
where $\mM_{k}$ is a symmetric matrix consists of 5 non-trivial parts
$\mP_{k}$, $\vq_{k}$, $\vr_{k}$, $\mS_{k}$, and $z_{k}$ as shown
in (\ref{eq:M}) in the main text. Following (\ref{eq:sp-x}), we
have the following decomposition for $\mM_{k}$
\begin{equation}
\mM_{k+1}-\mM_{k}=\tfrac{1}{n}\phi(\mM_{k})+(\mM_{k+1}-\EEk\mM_{k+1})+[\EEk\mM_{k+1}-\mM_{k}-\tfrac{1}{n}\phi(\mM_{k})],\label{eq:M-decomp}
\end{equation}
in which the matrix-valued function $\phi(\mM)$ represents the functions on the right hand sides of the ODE (\ref{eq:ODE}),
and $\EEk$ denotes the conditional expectation given the state of
the Markov chain $\mX_{k}$. Note that the stochastic process of
the macroscopic state $\mM_{k}$ is driven by the Markov chain of
the microscopic state $\mX_{k}$. Thus, $\EEk$ is well-defined. For future reference, we denotes $\EE$ the unconditional expectation of all the randomness of the Markov chain $\mX_k$, {\em i.e.}, 
the initial state $\mU,\mV_0,\vw_0$ and $\{\,\va_k,\vc_k,\tva_k,\tvc_k|k=0,1,2,\ldots \}$. By definition, $\sum_{k^{\prime}=0}^{k}(\mM_{k^{\prime}+1}-\EE_{k^{\prime}}\mM_{k^{\prime}})$
is a Martingale.

\subsection{Check the conditions provided in Lemma \ref{lem:base}}

In this subsection, we check the condition (C.1)--(C.5) for the decomposition
of (\ref{eq:M-decomp}). Once all conditions are proved to be satisfied, Theorem 1 will be proved. 

We first note that (C.5) is the assumption (A.5) in the main text.
Thus, (C.5) is satisfied. Before proving other conditions, we declare
a lemma.
\begin{lem}
\label{lem:helper}Under the same setting as Theorem 1, given $T>0$,
then
\begin{equation}
\EE\left(\sum_{\ell=1}^{d}[\mV_{k}]_{i,\ell}^{4}+[\vw_{k}]_{i}^{4}\right)\leq C(T)n^{-2},\quad\forall i=1,2,\ldots,n,\text{ and }k=0,1,\ldots,\left\lfloor nT\right\rfloor ,\label{eq:V4}
\end{equation}
\end{lem}
The proof can be founded in Section \ref{subsec:Proof-of-Lemma-helper}.

\subsubsection*{Check Condition (C.4)}
\begin{lem}
\label{lem:c4} Under the same setting as Theorem 1, for all $k=0,1,\ldots,\left\lfloor nT\right\rfloor $
with a given $T>0$, then
\begin{align*}
\EE\norm{\mP_{k}}^{2} & \leq C(T), & \quad &  & \EE\norm{\vq_{k}}^{2} & \leq C(T),\\
\EE\norm{\mS_{k}}^{2} & \leq C(T), &  &  & \EE z_{k}^{2} & \leq C(T),\\
\EE\norm{\vr_{k}}^{2} & \leq C(T).
\end{align*}
\end{lem}
\begin{proof}
It's a direct consequence of Lemma \ref{lem:helper}. We first verify
$\EE z_{k}^{2}\leq C(T).$ Using Holder's inequality, we have
\[
\EE z_{k}^{2}=\E\left({\textstyle \sum_{i=1}^{n}}w_{k,i}^{2}\right)^{2}\leq n\E{\textstyle \sum_{i=1}^{n}}w_{k,i}^{4}\leq C(T)
\]
For $[\mS_{k}]_{\ell,\ell}$, $\ell=1,\ldots,d$, similarly, we have
\[
\EE[\mS_{k}]_{\ell,\ell}^{2}=\E\left({\textstyle \sum_{i=1}^{n}}[\mV_{k}]_{i,\ell}^{2}\right)^{2}\leq C(T).
\]
and for $\EE[\mS_{k}]_{\ell,\ell^{\prime}}^{2}$, $\ell\neq\ell^{\prime}$,
we have:
\begin{align*}
\EE[\mS_{k}]_{\ell,\ell^{\prime}}^{2} & =\E\left({\textstyle \sum_{i=1}^{n}}[\mV_{k}]_{i,\ell}[\mV_{k}]_{i,\ell^{\prime}}\right)^{2}\\
 & \leq\E\left({\textstyle \sum_{i=1}^{n}}[\mV_{k}]_{i,\ell}^{2}\right)\left({\textstyle \sum_{i=1}^{n}}[\mV_{k}]_{i,\ell^{\prime}}^{2}\right)\\
 & \leq\sqrt{\E\left({\textstyle \sum_{i=1}^{n}}[\mV_{k}]_{i,\ell}^{2}\right)^{2}\E\left({\textstyle \sum_{i=1}^{n}}[\mV_{k}]_{i,\ell^{\prime}}^{2}\right)^{2}}\\
 & \leq C(T)
\end{align*}
where in reaching the third and last line, we used the Cauchy-Schwartz
inequality. Now, we get $\EE\norm{\mS_{k}}^{2}\leq C(T)$. The rest
bounds of $\EE\norm{\mP_{k}}^{2}$, $\EE\norm{\vq_{k}}^{2}$ and $\EE\norm{\vr_{k}}^{2}$
in Lemma \ref{lem:c4} can also be directly verified using the Cauchy-Schwartz
inequality.
\end{proof}

\subsubsection*{Check Condition (C.3)}
\begin{lem}
If Assumption (A.3) hold, $\phi(\mM)$ is a Lipschitz function.
\end{lem}
\begin{proof}
It suffices to verify each component of gradient $\nabla\phi(\mM)$
is bounded. Assumption (A.3) ensures that $H^{\prime}$ is Lipschitz
and the derivatives up to fourth order of the functions $f$, $\t f$
exists and uniformly bounded. These conditions guarantee that the
partial derivatives of $\phi(\mM)$ w.r.t. $\mP$, $\vq$, $\mS$
and $\vr$ are bounded. The remaining thing is to show that $\frac{\partial\phi(\mM)}{\partial z}$
is also bounded. Since there is a $\sqrt{z}$ term in $\phi(\mM)$,
the boundness can be potentially broken at $z=0$. However, we can
show that it is not the case. For example, we can show that $\bigl\langle\vc f(\vc^{\T}\vq+e\sqrt{z})\bigr\rangle_{\vc,e}$
is a Lipschitz function, because
\begin{align*}
\tfrac{\partial}{\partial z}\bigl\langle\vc f(\vc^{\T}\vq+e\sqrt{z})\bigr\rangle_{c,e} & =\tfrac{1}{2}z^{-\frac{1}{2}}\bigl\langle ecf^{\prime}(cq+e\sqrt{z})\bigr\rangle_{c,e}\\
 & =\tfrac{1}{2}\bigl\langle cf^{\prime\prime}(cq+e\sqrt{z})\bigr\rangle_{c,e}
\end{align*}
is always a well-defined bounded function. In reaching the first line,
we here interchanged the expectation and derivative, which is valid
because of the boundness of $f(\cdot)$, and in reaching the second
line, we used the Stein's lemma. Finally, other terms in (\ref{eq:def-ggb})
involving $\sqrt{z}$ can be treated in the same way. Thus, $\phi(\mM)$
is a Lipschitz function.
\end{proof}

\subsubsection*{Check Condition (C.2)}
\begin{lem}
\label{lem:c2}Under the same setting as Theorem 1, for all $k=0,1,\ldots,\left\lfloor nT\right\rfloor $
with a given $T>0$, then 
\[
\EE\norm{\EEk\mM_{k+1}-\mM_{k}-\tfrac{1}{n}\phi(\mM_{k})}\leq C(T)n^{-\frac{3}{2}}.
\]
\end{lem}
\begin{proof}
The above inequality can be split into 5 parts
\begin{align}
\EE\norm{\EEk\mP_{k+1}-\mP_{k}-\tfrac{\ttau}{n}(\vq_{k}\t{\vg}_{k}^{\T}+\mP_{k}\mL_{k})} & \leq C(T)n^{-\frac{3}{2}}\label{eq:bd-P}\\
\EE\norm{\EE_{k}\vq_{k+1}-\vq_{k}-\tfrac{\tau}{n}\left(\vg_{k}-\mP_{k}\t{\vg}_{k}+\vq_{k}h_{k}\right)} & \leq C(T)n^{-\frac{3}{2}}\label{eq:bd-q}\\
\EE\norm{\EE_{k}\mS_{k+1}-\mS_{k}-\tfrac{\ttau}{n}\left(\vr_{k}\t{\vg}_{k}^{\T}+\t{\vg}_{k}\vr_{k}^{\T}+\mS_{k}\mL_{k}+\mL_{k}\mS_{k}\right)} & \leq C(T)n^{-\frac{3}{2}}\label{eq:bd-S}\\
\EE\norm{\EE_{k}z_{k+1}-z_{k}-\tfrac{2\tau}{n}\left(\vq_{k}^{\T}\vg_{k}-\vr_{k}^{\T}\t{\vg}_{k}+z_{k}h_{k}\right)-\tfrac{\tau^{2}}{n}b_{k}} & \leq C(T)n^{-\frac{3}{2}},\label{eq:bd-z}\\
\EE\norm{\EE_{k}\vr_{k+1}-\vr_{k}-\tfrac{\tau}{n}\left(\mP_{k}^{\T}\vg_{k}-\mS_{k}\t{\vg}_{k}+\vr_{k}h_{k}\right)-\tfrac{\ttau}{n}\left(z_{k}\t{\vg}_{k}+\mL_{k}\vr_{k}\right)} & \leq C(T)n^{-\frac{3}{2}}\label{eq:bd-r}
\end{align}
where $\t{\vg}_{k},\;\mL_{k},\;\vg_{k},\;h_{k},\;b_{k}$ are defined
in (\ref{eq:def-tvgk}), (\ref{eq:def-Lk}), (\ref{eq:def-vgk}),
(\ref{eq:def-hk}) and (\ref{eq:def-bk}), respectively.

We first prove (\ref{eq:bd-P}). From (\ref{eq:update-V}), we have
\begin{equation}
\mV_{k+1}-\mV_{k}=\tfrac{\ttau}{n}\big[\vw_{k}\tvc_{2k+1}^{\T}\tf(\t{\vc}_{2k+1}^{\T}\mV_{k}^{\T}\vw_{k}+\eta_{\text{G}}\tva_{2k+1}^{\T}\vw_{k})-\lambda\mV_{k}\opdiag(H^{\prime}(\mS_k))\big].\label{eq:V-inc-d}
\end{equation}
Averaging both sides of the above equation over $\tvc_{2k+1}$ and
$\tva_{2k+1}$, we have
\begin{equation}
\EEk\mV_{k+1}-\mV_{k}=\tfrac{\ttau}{n}\big[\vw_{k}\t{\vg}_{k}^{\T}+\mV_{k}\mL_{k}\big],\label{eq:V-inc}
\end{equation}
where $\t{\vg}_{k}$ and  $\mL_k$ are defined in (\ref{eq:def-tvgk})
and (\ref{eq:def-Lk}), respectively. Multiplying $\mU^{\T}$ from
the left on the both sides of the above equation, we have 
\[
\EEk\mP_{k+1}-\mP_{k}=\tfrac{\ttau}{n}\left[\vq_{k}\t{\vg}_{k}^{\T}+\mP_{k}\mL_{k}\right],
\]
which implies (\ref{eq:bd-P}). In fact, there is no higher-order
term in (\ref{eq:bd-P}), and the left hand side of (\ref{eq:bd-P}) is exactly zero.

Then, we prove (\ref{eq:bd-q}). From (\ref{eq:update-w}), we have
\begin{equation}
\vw_{k+1}-\vw_{k}=\tfrac{\tau}{n}\left[\vy_{k}f(\vy_{k}^{\T}\vw_{k})-\tvy_{2k}\t f(\tvy_{2k}^{\T}\vw_{k})-\lambda\vw_{k}\opdiag(H^{\prime}(z_k))\right],\label{eq:w-inc-d}
\end{equation}
where $\vy_{k}=\mU\vc_{k}+\sqrt{\eta_{\text{T}}}\va_{k}$ and $\tvy_{2k}=\mV_{k}\tvc_{2k}+\sqrt{\eta_{\text{G}}}\tva_{2k}$.
Averaging both sides of the above equation over $\vc_{k}$, $\va_{k}$$\tvc_{2k}$
and $\tva_{2k}$, we have
\[
\begin{aligned}
\EEk\vw_{k+1}-\vw_{k}=&\tfrac{\tau}{n}\left[\mU\vg_{k}+\Big\langle \va_{k}f(\vc_{k}^{\T}\vq_{k}+\sqrt{\eta_{\text{T}}}\va_{k}^{\T}\vw_{k})\right\rangle 
\\
&-\mV_{k}\t{\vg}_{k}-\left\langle \tva_{2k}\tf(\tvc_{2k}^{\T}\vr_{k}+\sqrt{\eta_{\text{G}}}\tva_{2k}^{\T}\vw_{k})\right\rangle -\lambda\vw_{k}\opdiag(H^{\prime}(z_k))\Big].
\end{aligned}
\]
Multiplying $\mU^{\T}$ from the left on the both sides of the above
equation, we have 
\begin{align}
\EEk\vq_{k+1}-\vq_{k}=\tfrac{\tau}{n} & \Big[\vg_{k}-\mP_{k}\t{\vg}_{k}+\sqrt{\eta_{\text{T}}}\left\langle \mU^{\T}\va_{k}f(\vc_{k}^{\T}\vq_{k}+\sqrt{\eta_{\text{T}}}\va_{k}^{\T}\vw_{k})\right\rangle _{\vc,\va}\nonumber \\
 & -\sqrt{\eta_{\text{G}}} \left\langle \mU^{\T}\tva\tf(\tvc^{\T}\vr_{k}+\sqrt{\eta_{\text{G}}}\tva^{\T}\vw_{k})\right\rangle _{\tvc,\tva}-\lambda\vq_{k}\opdiag(H^{\prime}(z_k))\Big]\label{eq:q-inc}
\end{align}
We note that 
$\begin{bmatrix}\mU^{\T}\va_{k}\\
\vw_{k}^{\T}\va_{k}
\end{bmatrix}$ are Gaussian random vector with zero-mean and covariance matrix $\begin{bmatrix}\mI & \vq_{k}\\
\vq_{k}^{\T} & z_{k}
\end{bmatrix}$. We can rewrite
\begin{align}
\left\langle \mU^{\T}\va f(\vc^{\T}\vq_{k}+\sqrt{\eta_{\text{T}}}\va^{\T}\vw_{k})\right\rangle _{\vc,\va} & =z_{k}^{-1/2}\mU^{\T}\vw_{k}\left\langle ef(\vc^{\T}\vq_{k}+\sqrt{z_{k}\eta_{\text{T}}}e\right\rangle _{\vc,e}\label{eq:fp}\\
 & =\sqrt{\eta_{T}}\vq_{k}\left\langle f^{\prime}(\vc^{\T}\vq_{k}+\sqrt{z_{k}\eta_{\text{T}}}e\right\rangle _{\vc,e},\nonumber 
\end{align}
where the second line is due to Stein's lemma (i.e., integral by part
for Gaussian random variable.) Similarly, we have 
\begin{equation}
\left\langle \mU^{\T}\tva\tf(\tvc^{\T}\vr_{k}+\sqrt{\eta_{\text{G}}}\tva^{\T}\vw_{k})\right\rangle _{\tvc,\tva}=\sqrt{\eta_{\text{G}}} \vq_{k}\left\langle \tf^{\prime}(\t{\vc}^{\T}\vr_{k}+\sqrt{z_{k}\eta_{\text{G}}}e\right\rangle _{\tvc,e}.\label{eq:tfp}
\end{equation}
Substituting (\ref{eq:fp}) and (\ref{eq:tfp}) into (\ref{eq:q-inc}),
we get
\[
\EEk\vq_{k+1}-\vq_{k}=\tfrac{\tau}{n}\left[\vg_{k}-\mP_{k}\t{\vg}_{k}+\vq_{k}h_{k}\right],
\]
where $\t{\vg}_{k}$, $\vg_{k}$, and $h_{k}$ are defined in (\ref{eq:def-tvgk}),
(\ref{eq:def-vgk}), and (\ref{eq:def-hk}), respectively. Now, we
proved (\ref{eq:bd-q}), which again has no higher-order term.

We next prove (\ref{eq:bd-S}). Note that 
\begin{align*}
\mS_{k+1}-\mS_{k} & =(\mV_{k}+\mV_{k+1}-\mV_{k})^{\T}(\mV_{k}+\mV_{k+1}-\mV_{k})-\mS_{k}\\
 & =\mV_{k}^{\T}\left(\mV_{k+1}-\mV_{k}\right)+\left(\mV_{k+1}-\mV_{k}\right)^{\T}\mV_{k}+\left(\mV_{k+1}-\mV_{k}\right)^{\T}\left(\mV_{k+1}-\mV_{k}\right).
\end{align*}
Averaging both sides of the above equation over $\tvc_{2k+1}$ and
$\tva_{2k+1}$ and substituting (\ref{eq:V-inc}) into above equation,
we have 
\begin{equation}
\EEk\mS_{k+1}-\mS_{k}=\tfrac{\ttau}{n}\left[\vr_{k}\t{\vg}_{k}^{\T}+\mS_{k}\mL_{k}+\t{\vg}_{k}\vr_{k}^{\T}+\mL_{k}\mS_{k}\right]+\tfrac{\ttau^{2}}{n^{2}}\big[\vw_{k}\t{\vg}_{k}^{\T}+\mV_{k}\mL_{k}\big]^{\T}\big[\vw_{k}\t{\vg}_{k}^{\T}+\mV_{k}\mL_{k}].\label{eq:S-inc}
\end{equation}
We know that
\begin{align}
\EE\norm{\big[\vw_{k}\t{\vg}_{k}^{\T}+\mV_{k}\mL_{k}\big]^{\T}\big[\vw_{k}\t{\vg}_{k}^{\T}+\mV_{k}\mL_{k}]} & \leq\EE\norm{\vw_{k}\t{\vg}_{k}^{\T}+\mV_{k}\mL_{k}}^{2}\nonumber \\
 & \leq2z_{k}\norm{\t{\vg}_{k}}^{2}+2\norm{\mS_{k}}\norm{\mL_{k}}^{2}\nonumber \\
 & \leq C\EE\left[z_{k}+\norm{\mS_{k}}\right]\nonumber \\
 & \leq C(T),\label{eq:wgvl}
\end{align}
where $\t{\vg}_{k}$, $\mL_{k}$ are defined in (\ref{eq:def-tvgk})
and (\ref{eq:def-Lk}), respectively. The third line of the above
inequalities is due to the fact that  $\tf$ and $H^{\prime}$ are uniformly bounded,
and in reaching the last line, we used Lemma~\ref{lem:c4}. Combining
(\ref{eq:S-inc}) and (\ref{eq:wgvl}), we reach (\ref{eq:bd-S}).

The other two inequalities (\ref{eq:bd-z}) and (\ref{eq:bd-r}) can
be proved in a similar way. We omit the details here.
\end{proof}

\subsubsection*{Check Condition (C.1)}
\begin{lem}
Under the same setting as Theorem 1, for all $k=0,1,\ldots,\left\lfloor nT\right\rfloor $
with a given $T>0$, then 
\[
\EE\norm{\mM_{k+1}-\EEk\mM_{k+1}}^{2}\leq C(T)n^{-2}.
\]
\end{lem}
\begin{proof}
Note that $\EE\norm{\mM_{k+1}-\EEk\mM_{k+1}}^{2}=\EE\norm{\mM_{k+1}-\mM_{k}-\EEk(\mM_{k+1}-\mM_{k})}^{2}\leq\EE\norm{\mM_{k+1}-\mM_{k}}^{2}$.
It is sufficient to prove
\begin{equation}
\EE\norm{\mM_{k+1}-\mM_{k}}^{2}\leq C(T)n^{-2}.\label{eq:bd-M2}
\end{equation}
In what follows, we are going to bound the second-order moment of
each element in $\mM_{k+1}-\mM_{k}$. In particular, we bound the
5 blocks $\mP_{k}$, $\mS_{k}$, $\vq_{k}$, $z_{k}$ and $\vr_{k}$
of $\mM_{k}$ separately.

We first bound $\EE\norm{\mP_{k+1}-\mP_{k}}^{2}$. Multiplying $\mU^{\T}$
from left on both sides of (\ref{eq:V-inc-d}), we have 
\[
\mP_{k+1}-\mP_{k}=\tfrac{\ttau}{n}\big[\vq_{k}\tvc_{2k+1}^{\T}\tf(\t{\vc}_{2k+1}^{\T}\mV_{k}^{\T}\vw_{k}+\eta_{\text{G}}\tva_{2k+1}^{\T}\vw_{k})-\lambda\mP_{k}\opdiag(H^{\prime}(\mV_{k}^{\T}\mV_{k}))\big]
\]
We then get 
\begin{align}
\EE\norm{\mP_{k+1}-\mP_{k}}^{2} & \leq Cn^{-2}\EE\Big[\norm{\vq_{k}}^{2}\EEk\norm{\t{\vc}_{2k+1}}^{2}+\norm{\mP_{k}}^{2}\Big]\nonumber \\
 & \leq Cn^{-2}\EE\left[1+\norm{\vq_{k}}^{2}+\norm{\mP_{k}}^{2}\right]\nonumber \\
 & \leq C(T)n^{-2}.\label{eq:bd-P2}
\end{align}
Here the last line is due to Lemma \ref{lem:c4}.

We next bound $\EE\norm{\vq_{k+1}-\vq_{k}}^{2}$ in the same way.
Specifically, multiplying $\mU^{\T}$ from the left on both sides
of (\ref{eq:w-inc-d}), we get

\[
\vq_{k+1}-\vq_{k}=\tfrac{\tau}{n}\left[\mU^{\T}\vy_{k}f(\vy_{k}^{\T}\vw_{k})-\mU^{\T}\tvy_{2k}\t f(\tvy_{2k}^{\T}\vw_{k})-\lambda\vq_{k}\opdiag(H^{\prime}(\vw_{k}^{\T}\vw_{k}))\right].
\]
We then have
\begin{align}
&\EE\norm{\vq_{k+1}-\vq_{k}}^{2}  \nonumber\\
&\leq\tfrac{\tau^{2}}{n^{2}}\EE\bigl[\norm{\vc_{k}}^{2}f_{k}^{2}+\norm{\mU^{\T}\va_{k}}^{2}f_{k}^{2}+\norm{\mP_{k}}^{2}\norm{\t{\vc}_{2k}}^{2}\tf_{2k}^{2}+\norm{\mU^{\T}\tva_{2k}}^{2}\tf_{2k}^{2}+\norm{\vq_{k}}^{2}h_{k}^{2}\bigr]\nonumber \\
 & \leq Cn^{-2}\bigl[1+\sqrt{\EE\norm{\mU^{\T}\va_{k}}^{4}}\sqrt{\EE f_{k}^{4}}+\sqrt{\EE\norm{\mU^{\T}\tva_{2k}}^{4}}\sqrt{\EE\tf_{2k}^{4}}+\EE z_{k}^{2}+\EE\norm{\mS_{k}}^{2}\bigr]\nonumber \\
 & \leq Cn^{-2}[1+\EE z_{k}^{2}+\EE\norm{\mS_{k}}^{2}]\nonumber \\
 & \leq C(T)n^{-2},\label{eq:bd-q2}
\end{align}
where $f_{k}$ and $\tf_{2k}$ are shorthands for $f(\vy_{k}^{\T}\vw_{k})$
and $\tf(\tvy_{2k}^{\T}\vw_{k})$ respectively. In reaching the last
line, we used Lemma \ref{lem:c4} again.

Similarly, we can also prove that 
\begin{equation}
\begin{aligned}\EE\norm{\mS_{k+1}-\mS_{k}}^{2} & \leq C(T)n^{-2}\\
\EE(z_{k+1}-z_{k})^{2} & \leq C(T)n^{-2}\\
\EE\norm{\vr_{k+1}-\vr_{k}}^{2} & \leq C(T)n^{-2}.
\end{aligned}
\label{eq:bd-Szr2}
\end{equation}
Combining (\ref{eq:bd-P2}), (\ref{eq:bd-q2}) and (\ref{eq:bd-Szr2}),
we can prove (\ref{eq:bd-M2}), which concludes the whole proof.
\end{proof}

\subsection{Proof of Lemma \ref{lem:helper}\label{subsec:Proof-of-Lemma-helper}}

Before proving Lemma \ref{lem:helper}, we first present and prove
the following lemma.
Let $\vu_i$ and $\vv_{k,i}$ denote the $i$th row vectors of $\mU$ and $\mV_k$ in column view, respectively, and let $w_{k,i}$ be the $i$th element of the vector $\vw_k$. 
\begin{lem}
Under the same setting as Theorem 1, for all $k=0,1,\ldots,\left\lfloor nT\right\rfloor $
with a given $T>0$, then
\begin{align}
\norm{\EEk\vv_{k+1,i}-\vv_{k,i}} & \leq Cn^{-1}\left(\norm{\vv_{k,i}}+\abs{w_{k,i}}\right)\label{eq:bd-v}\\
\abs{\EEk w_{k,i}-w_{k,i}} & \leq Cn^{-1}\left(\norm{\vu_{i}}+\norm{\vv_{k,i}}+\abs{w_{k,i}}\right).\label{eq:bd-w}
\end{align}
\end{lem}
In the proof of this lemma and Lemma \ref{lem:helper}, we omit the two constants $\eta_\text{T}$ and $\eta_\text{G}$ for simplicity.
\begin{proof}
From (\ref{eq:update-V}) and knowing that the function $\tf$ and
$H^{\prime}$ are uniformly bounded, we can immediately prove \eqref{eq:bd-v}.

Next, we are going to prove (\ref{eq:bd-w}). From (\ref{eq:update-w}),
we know
\begin{align}
&\abs{\EEk w_{k+1,i}-w_{k,i}} \nonumber
\\
& \leq  \tfrac{\tau}{n}\bigg(\abs{\vu_{i}^{\T}\left\langle \vc_{k}f(\vy_{k}^{\T}\vw_{k})\right\rangle _{\vc_{k},\va_{k}}}+\abs{\left\langle a_{k,i}f(\vy_{k}^{\T}\vw_{k})\right\rangle _{\vc_{k},\va_{k}}}\nonumber \\
 & +\abs{\vv_{k,i}^{\T}\left\langle \t{\vc}_{2k}\tf(\tvy_{2k}^{\T}\vw_{k})\right\rangle _{\tvc_{2k},\tva_{2k}}}+\abs{\left\langle \t a_{2k,i}\tf(\tvy_{2k}^{\T}\vw_{k})\right\rangle _{\tvc_{2k},\tva_{2k}}}+\lambda\abs{w_{k,i}H^{\prime}(\vw_{k}^{\T}\vw_{k})}\bigg)\nonumber \\
\leq & Cn^{-1}\bigg(\norm{\vu_{i}}+\norm{\vv_{k,i}}+\abs{w_{k,i}}+\abs{\left\langle a_{k,i}f(\vy_{k}^{\T}\vw_{k})\right\rangle _{\vc_{k},\va_{k}}}+\abs{\left\langle \t a_{2k,i}\tf(\tvy_{2k}^{\T}\vw_{k})\right\rangle _{\tvc_{2k},\tva_{2k}}}\bigg),\label{eq:ineq-w}
\end{align}
where the last is due to the fact that $H^{\prime}$, $f$ and $\t f$
are uniformly bounded. Using Taylor's expansion up-to zero-order
\begin{align*}
f(\vy_{k}^{\T}\vw_{k}) & =f(\vq_{k}^{\T}\vc_{k}+{\textstyle \sum}_{j\neq i}w_{k,j}a_{k,j}+w_{k,j}a_{k,j})\\
 & =f(\vq_{k}^{\T}\vc_{k}+{\textstyle \sum}_{j\neq i}w_{k,j}a_{k,j})+f^{\prime}(\vq_{k}^{\T}\vc_{k}+{\textstyle \sum}_{j\neq i}w_{k,j}a_{k,j}+\chi_{k,i})w_{k,j}a_{k,j},
\end{align*}
with $\chi_{k,i}$ being some number such that $\abs{\chi_{k,i}}\leq\abs{w_{k,i}a_{k,i}},$
we have
\begin{align}
& \abs{\left\langle a_{k,i}f(\vy_{k}^{\T}\vw_{k})\right\rangle _{\vc_{k},\va_{k}}}  \nonumber
\\
& \leq\abs{\left\langle f(\vq_{k}^{\T}\vc_{k}+{\textstyle \sum}_{j\neq i}w_{k,j}a_{k,j})a_{k,i}\right\rangle _{\vc_{k},\va_{k}}}+\abs{\left\langle f^{\prime}(\vq_{k}^{\T}\vc_{k}+{\textstyle \sum}_{j\neq i}w_{k,j}a_{k,j}+\chi_{k,i})w_{k,j}a_{k,j}^{2}\right\rangle _{\vc_{k},\va_{k}}}\nonumber \\
 & =\abs{\left\langle f^{\prime}(\vq_{k}^{\T}\vc_{k}+{\textstyle \sum}_{j\neq i}w_{k,j}a_{k,j}+\chi_{k,i})w_{k,i}a_{k,i}^{2}\right\rangle _{\vc_{k},\va_{k}}}\nonumber \\
 & \leq C\abs{w_{k,i}}.\label{eq:ineq-af}
\end{align}
The second line is due to the fact $a_{k,i}$ is zero-mean, and in
reaching the last line, we used the boundness of $f^{\prime}$. Similarly,
we can get 
\begin{equation}
\abs{\left\langle \t a_{2k,i}\tf(\tvy_{2k}^{\T}\vw_{k})\right\rangle _{\tvc_{2k},\tva_{2k}}}\leq C\abs{w_{k,i}}.\label{eq:ineq-atf}
\end{equation}
Substituting (\ref{eq:ineq-af}) and (\ref{eq:ineq-atf}) into (\ref{eq:ineq-w}),
we prove (\ref{eq:ineq-w}).
\end{proof}
Now we are in the position to prove Lemma \ref{lem:helper}.
\begin{proof}[Proof of Lemma \ref{lem:helper}]
Because of the exchangeability, $\EE w_{k,i}^{4}=\EE w_{k,j}^{4}$,
and $\EE[\mV_{k}]_{i,\ell}^{4}=\EE[\mV_{k}]_{j,\ell}^{4}$ for all
$i,j=1,2,\ldots,n$ and $\ell=1,2,\ldots,d$. Thus, we only need to
prove (\ref{eq:V4}) for any specific $i$.

We first prove $\EE w_{k,i}^{4}\leq C(T)n^{-2}$. We know that
\begin{align}
\EE w_{k+1,i}^{4}-\EE w_{k,i}^{4} & =4\EE\left[w_{k,i}^{3}\EEk\left(w_{k+1,i}-w_{k,i}\right)\right]+6\EE\left[w_{k,i}^{2}\EEk\left(w_{k+1,i}-w_{k,i}\right)^{2}\right]\label{eq:w4}\\
 & \;+4\EE\left[w_{k,i}\EEk\left(w_{k+1,i}-w_{k,i}\right)^{3}\right]+\E\EEk\left(w_{k+1,i}-w_{k,i}\right)^{4}.\nonumber 
\end{align}
 From (\ref{eq:update-w}) and knowing that $h$, $f$ and $\tf$
are uniformly bounded, we have 
\begin{equation}
\EEk\left(w_{k+1,i}-w_{k,i}\right)^{\gamma}\leq\frac{C}{n^{\gamma}}\left(1+\norm{\vu_{i}}^{\gamma}+\norm{\vv_{k,i}}^{\gamma}+\abs{w_{k,i}}^{\gamma}\right)\quad\text{for }\gamma=2,3,4.\label{eq:w-d}
\end{equation}
Substituting (\ref{eq:bd-w}) and (\ref{eq:w-d}) into (\ref{eq:w4})
and using the Young's inequality, we have
\begin{align}
\EE w_{k+1,i}^{4}-\EE w_{k,i}^{4} & \leq\tfrac{C}{n}\left(n^{-2}+\EE\norm{\vu_{i}}^{4}+\EE\norm{\vv_{k,i}}^{4}+\EE w_{k,i}^{4}\right).\nonumber \\
 & \leq\tfrac{C}{n}\EE\left(n^{-2}+{\textstyle \sum_{\ell=1}^{d}}[\mV_{k}]_{i,\ell}^{4}+w_{k,i}^{4}\right),\label{eq:ineq-w4}
\end{align}
where the last line is due to Assumption A.4), which implies $\sum_{\ell}[\mU]_{i,\ell}^{4}\leq C$.
Similarly, we can prove 
\begin{equation}
{\textstyle \sum_{\ell=1}^{d}}\EE\left([\mV_{k+1}]_{i,\ell}^{4}-[\mV_{k}]_{i,\ell}^{4}\right)\leq\tfrac{C}{n}\EE\left(n^{-2}+{\textstyle \sum_{\ell=1}^{d}}[\mV_{k}]_{i,\ell}^{4}+w_{k,i}^{4}\right).\label{eq:ineq-v4}
\end{equation}
Combining (\ref{eq:ineq-w4}) and (\ref{eq:ineq-v4}), we have 
\[
\EE(w_{k+1,i}^{4}+{\textstyle \sum_{\ell=1}^{d}}[\mV_{k+1}]_{i,\ell}^{4})-\EE\left(w_{k,i}^{4}+{\textstyle \sum_{\ell=1}^{d}}[\mV_{k}]_{i,\ell}^{4}\right)\leq\tfrac{C}{n}\left[n^{-2}+\EE\left(w_{k,i}^{4}+{\textstyle \sum_{\ell=1}^{d}}[\mV_{k}]_{i,\ell}^{4}\right)\right].
\]
Using the above inequality iteratively, we have
\[
\EE\left(w_{k,i}^{4}+{\textstyle \sum_{\ell=1}^{d}}[\mV_{k}]_{i,\ell}^{4}\right)\leq\left(n^{-2}+w_{0,i}^{4}+{\textstyle \sum_{\ell=1}^{d}}[\mV_{0}]_{i,\ell}^{4}\right)e^{\frac{k}{n}C}.
\]
Since $\EE\big(w_{0,i}^{4}+{\textstyle \sum_{\ell=1}^{d}}[\mV_{0}]_{i,\ell}^{4}\big)$
are bounded in Assumption A.4), we now reach (\ref{eq:V4}).
\end{proof}

\section{Local stability analysis of the fixed points of the ODE}
In this section, we provide additional details on the local stability analysis of the ODE  for Example 1\label{sec:stability}.
We first its simplified ODE \eqref{eq:PCA-ODE} in the main text. Then, we provide the derivation of the local stability analysis when $d=1$, where the main results are summarized in Section \ref{sec:phase}.  Finally, we establish the proof of Claim 1 in the main text.
\subsection{Derive the reduced ODE for Example \ref{ex:PCA} when $\lambda\to\infty$ }
In Example  \ref{ex:PCA}, $f(x)=\t f(x) = x$. Plugging back to (\ref{eq:def-ggb}), we obtain that
\begin{equation}
\label{eq:def-ggb-WGAN}
\begin{aligned}
\vg_t&=\mat{\Lambda}\vq_t
\\
\t\vg_t&= \mat{\t\Lambda}\vr_t
\\
b_t&=\eta_{\text{T}}(\vq_t^\T\mat{\Lambda}\vq_t+\eta_{\text{T}}z_t)+\eta_{\text{G}}(\vr_t^\T\mat{\t\Lambda}\vr_t+\eta_{\text{G}}z_t).
\end{aligned}
\end{equation}
Correspondingly, ODE in (\ref{eq:ODE}) becomes:
\begin{equation} \label{eq:ODE-WGAN}
\begin{aligned}
\tfrac{\dif}{\dif t} \mP_t &=
\t\tau \big( \vq_t \t\vr_t^\T\mat{\t\Lambda} + \mP_t \mL_t  \big)
\\
\tfrac{\dif}{\dif t} \vq_t &= \tau \big( \mat{\Lambda}\vq_t - \mP_t \mat{\t\Lambda}\vr_t+ \vq_t h_t \big)
\\
\tfrac{\dif}{\dif t} \vr_t &=\tau \big( \mP_t^T \mat{\Lambda}\vq_t - \mS_t \mat{\t\Lambda}\vr_t + \vr_t h_t \big)
 + \t\tau \big( \mat{\t\Lambda}\vr_t + \mL_t\vr_t  \big)
\\
\tfrac{\dif}{\dif t} \mS_t &= \t\tau \big( \vr_t \vr_t^\T \mat{\t\Lambda}^\T + \mat{\t\Lambda}\vr_t \vr_t^\T
 + \mS_t \mL_t + \mL_t\mS_t\big)
\\
\tfrac{\dif}{\dif t} z_t &= 2\tau ( \vq_t^\T \mat{\Lambda}\vq_t - \vr_t^\T \mat{\t\Lambda}\vr_t  + z_t h_t ) \\
&\quad+ \tau^2 [\eta_{\text{T}}(\vq_t^\T\mat{\Lambda}\vq_t+z_t\eta_{\text{T}})+\eta_{\text{G}}(\vr_t^\T\mat{\t\Lambda}\vr_t+z_t\eta_{\text{G}})]
\end{aligned}
\end{equation}
The first four equations are exactly (\ref{eq:PCA-ODE}). From last two equations of (\ref{eq:ODE-WGAN}), by setting $\tfrac{\dif}{\dif t} \opdiag\{\mS_t\} = \mathbf{0}$, $\tfrac{\dif}{\dif t} z_t = 0$, $\opdiag(\mS_t)= \mI$ and $z_t=1$, we can get (\ref{eq:h-mL}).

\subsection{A complete study of all fixed points when $d=1$}
\label{sec:fiexed-d-1}

We next provide the local stability analysis of the fixed points of
the ODE \eqref{eq:PCA-ODE}. 
When $d=1$ and $\lambda\to\infty$, the macroscopic state is described by only 3 scalars, $P_t$, $q_t$ and $r_t$. The result
is summarized in Table 1. 
For the sake of simplicity,  we only consider the case $\Lambda=\t\Lambda$, and set $\eta_\text{T}=\eta_\text{G}=1$, but all analysis can be extended to general cases.

The fixed points are given by the condition $ \frac{d}{dt} P_{t}= \frac{d}{dt}q_{t}= \frac{d}{dt}r_{t}=0$.
From \eqref{eq:PCA-ODE}, we get

\begin{equation}
\begin{cases}
\ttau \Lambda r \left(q-r P \right)=0 \\
\tau \big[ \Lambda - \tau- \Lambda  \left(1+ \tfrac{ \tau}{2} \right)q^{2} \big]q- \tau \Lambda  \big[ P+ \left( \tfrac{ \tau}{2}-1 \right)rq \big]r=0 \\
\tau \Lambda  P q+ \bigl[ \Lambda ( \ttau- \tau)- \tau^{2} \bigr]r+ \Lambda  \big( \tau- \ttau- \tfrac{ \tau^{2}\}}{2} \big)r^{3}- \tau \Lambda  \big(1+ \tfrac{ \tau}{2} \big)rq^{2}=0,
\end{cases} \label{eq:fixed-eq}
\end{equation}
where $P, q, r$ are the stationary macroscopic state.
The local stability of a fixed point is identified by whether the
Jacobian matrix
\[
J( P,q,r) \bydef \begin{bmatrix} \frac{ \partial}{ \partial P}g_{1} & \frac{ \partial}{ \partial q}g_{1} & \frac{ \partial}{ \partial r}g_{1} \\
\frac{ \partial}{ \partial P}g_{3} & \frac{ \partial}{ \partial q}g_{3} & \frac{ \partial}{ \partial r}g_{3} \\
\frac{ \partial}{ \partial P}g_{5} & \frac{ \partial}{ \partial q}g_{5} & \frac{ \partial}{ \partial r}g_{5}
\end{bmatrix}
\]
has eigenvalue with non-negative real part or not, where $g_{1}= \ttau \Lambda r \left(q-r P \right)$,
$g_{2}= \tau \left[ \Lambda - \tau- \Lambda  \left(1+ \tfrac{ \tau}{2} \right)q^{2} \right]q- \tau \Lambda  \left[ P+ \left( \tfrac{ \tau}{2}-1 \right)rq \right]r$
and $g_{5}= \tau\Lambda  P q+ \bigl[ \Lambda ( \ttau- \tau)- \frac{(\eta_{\text{T}}+\eta_{\text{G}})\tau^{2}}{2} \bigr]r+ \Lambda  \left( \tau- \ttau- \tfrac{ \tau^{2}\eta_{\text{G}}}{2} \right)r^{3}- \tau \Lambda  \left(1+ \tfrac{ \tau \eta_{\text{T}}}{2} \right)rq^{2}$.

\subsubsection*{Type (1) fixed point at $ \protect P=q=r=0$}

It is easy to verify that $q=r=0$ and any $ P \in[-1,1]$ is a solution
of (\ref{eq:fixed-eq}), but we first consider $ P=0$.

The Jacobian at $ P=q=r=0$ is
\[
J(0,0,0)= \begin{bmatrix}0 & 0 & 0 \\
0 & \tau( \Lambda - \tau) & 0 \\
0 & 0 & \Lambda  \left( \ttau- \tau \right)- \tau^{2}
\end{bmatrix}.
\]
Thus, type (1) fixed point is stable if and only if

\[
\tau \geq \Lambda  \quad \text{ and } \quad \tfrac{ \ttau}{ \tau} \leq \tfrac{ \tau+ \Lambda }{\Lambda }.
\]

\subsubsection*{Type (2) fixed points at $ \protect P=q=0$, $r= \pm r^{ \ast} \protect \neq0$}

We first analyze when such fixed point exists and then study its local
stability.

If $ P=q=0$, the first two equations in (\ref{eq:fixed-eq}) trivially
hold. The third equation becomes

\[
\tau[ \Lambda (r^{2}-1)- \tfrac{ \tau}{2}( \Lambda r^{2}+2)]- \t{ \tau} \Lambda (r^{2}-1)=0.
\]
The solution is
\begin{equation}
r^{2}= \frac{ \tau- \t{ \tau}+ \tau^{2}/ \Lambda }{ \tau- \t{ \tau}- \tau^{2}/2}. \label{eq:r_equation}
\end{equation}
Since only the positive solution corresponds a fixed one. Thus, type
(2) fixed point exists if
\begin{align}
 & \tfrac{ \ttau}{ \tau} \leq1- \tfrac{ \tau}{2} \label{eq:r-exist-1} \\
\text{ or }\quad & \tfrac{ \ttau}{ \tau} \geq \tfrac{ \tau+ \Lambda }{ \Lambda }. \label{eq:r-exist-2}
\end{align}

Next, we investigate the local stability of this fixed point. The
Jacobian at $ \tilde{q}=q=0$ for a given $r$ is

\begin{equation}
J(0,0,r)= \begin{bmatrix}- \t{ \tau} \Lambda r^{2} & \t{ \tau} \Lambda r & 0 \\
- \tau \Lambda r & \tau( \Lambda - \tau)- \Lambda  \tau( \frac{ \tau}{2}-1)r^{2} & 0 \\
0 & 0 & 3r^{2} \Lambda ( \tau- \frac{ \tau^{2}}{2}- \t{ \tau})- \tau^{2}+ \Lambda ( \t{ \tau}- \tau)
\end{bmatrix} \label{eq:Jtildeq2}
\end{equation}

Plugging (\ref{eq:r_equation}) into $[J(0,0,r)]_{3,3}$ of (\ref{eq:Jtildeq2}),
then $[J(0,0,r)]_{3,3} \leq0$ implies
\[
\tfrac{ \ttau}{ \tau} \geq \tfrac{ \tau}{\Lambda }+1.
\]
 It indicates that the stationary points at the region (\ref{eq:r-exist-1})
are always unstable. Thus, we only need to consider the second region
specified by (\ref{eq:r-exist-2}).

For the upper-left $2 \times2$ sub-matrix of (\ref{eq:Jtildeq2}),
the eigenvalues are non-positive if and only if
\begin{align}
- \t{ \tau} \Lambda r^{2}+ \tau( \Lambda - \tau)- \Lambda  \tau( \tfrac{ \tau}{2}-1)r^{2} & \leq0 \label{eq:a+d1} \\
\tau+ \Lambda ( \tfrac{ \tau}{2}-1)r^{2}+ \Lambda - \Lambda  & \geq0. \label{eq:ad-bc1}
\end{align}
Plugging (\ref{eq:r_equation}) into (\ref{eq:a+d1}), we can get
\begin{equation}
\tfrac{ \ttau}{ \tau} \geq2. \label{eq:boundary_1}
\end{equation}
Plugging (\ref{eq:r_equation}) into (\ref{eq:ad-bc1}) and combining
(\ref{eq:r-exist-2}), we can get

\[
[ \tau+ \Lambda ( \tfrac{ \tau}{2}-1)] \tilde{ \tau} \geq \tau \Lambda ( \tfrac{ \tau}{2}-1).
\]
Solving this inequality implies that
\begin{equation}
\frac{ \tilde{ \tau}}{ \tau} \leq \frac{( \tfrac{ \tau}{2}-1) \Lambda }{( \frac{ \tau}{2}-1) \Lambda + \tau}, \text{ when } \tau< \frac{2 \Lambda }{ \Lambda +2} \label{eq:boundary_2}
\end{equation}
and
\begin{equation}
\frac{ \tilde{ \tau}}{ \tau} \geq \frac{( \frac{ \tau}{2}-1) \Lambda }{( \frac{ \tau}{2}-1) \Lambda + \tau}, \, \text{ when } \tau> \frac{2 \Lambda }{ \Lambda +2}. \label{eq:boundary_3}
\end{equation}
Note that (\ref{eq:boundary_3}) is included by (\ref{eq:boundary_1}),
as $ \frac{( \frac{ \tau}{2}-1) \Lambda }{( \frac{ \tau}{2}-1) \Lambda + \tau} \leq2$
when $ \tau> \frac{2 \Lambda }{ \Lambda +2}$.

Then, combining (\ref{eq:r-exist-2}), (\ref{eq:boundary_1}), and
(\ref{eq:boundary_2}) we obtain the stability region for $ \tilde{q}=q=0$,
\[
\frac{ \tilde{ \tau}}{ \tau} \geq1+ \frac{ \tau}{ \Lambda }, \, \frac{ \tilde{ \tau}}{ \tau} \geq2, \, \text{and } \frac{ \tilde{ \tau}}{ \tau} \leq \beta( \tau),
\]
where $\beta( \tau)$ is defined as
\[
\beta(\tau) \bydef \begin{cases}
\frac{( \frac{ \tau}{2}-1) \Lambda }{( \frac{ \tau}{2}-1) \Lambda + \tau} & \text{ if } \tau \leq \frac{2 \Lambda }{ \Lambda +2} \\
+ \infty & \text{ otherwise.}
\end{cases}
\]

\subsubsection*{Type (3) fixed points at $q=r=0$ and $ \protect \abs{ \protect P} \in(0,1]$}

As mentioned, we can check that $q=r=0$ and any $ P \in[-1,1]$ is
a solution of (\ref{eq:fixed-eq}). We next investigate the stable
region for the fixed point $ P= \pm1$ and $q=r=0$, which represents the perfect recovery state. 
For general $P$, we can analyze its fixed point similarly.

The Jacobian at $q=r=0$ for any given $ P$ is

\[
J(1,0,0)= \begin{bmatrix}0 & 0 & 0 \\
0 & \tau( \Lambda - \tau) & - \tau \Lambda  \\
0 & \tau \Lambda  & \Lambda  \left( \ttau- \tau \right)- \tau^{2}
\end{bmatrix}.
\]

In this case, $J(1,0,0)$ always has an eigenvalue $0$ and to calculate
the rest two eigenvalues, we only need to analyze the bottom-right
$2 \times2$ sub-matrix of $J( \tilde{q})$. The characteristic polynomial
of this sub-matrix is $f( \lambda)= \lambda^{2}-(a+d) \lambda+ad-bc,$where
$a= \tau( \Lambda - \tau),$ $b=- \tau \Lambda ,$ $c= \tau \Lambda ,$
and $d= \Lambda  \left( \ttau- \tau \right)- \tau^{2}$. The roots of
$f( \lambda)=0$ both have non-positive real part if and only if $a+d \leq0, \,ad-bc \geq0$,
which implies
\begin{equation}
\tfrac{ \ttau}{ \tau} \leq \tfrac{2 \tau}{ \Lambda }
\quad\text{ and }\quad
  \tfrac{ \t{ \tau}}{ \tau}( \tau- \Lambda )  \leq \tfrac{\tau^2}{\Lambda}.
 \label{eq:tq-1}
\end{equation}
Noting that when $\tau < \Lambda$, the second inequality always hold, and when $\tau > \Lambda$, 
$\tfrac{\tau^2}{\Lambda(\tau-\Lambda)} \geq 4$, we can combine the two inequalities in \eqref{eq:tq-1}
into  compact form
\[
\tfrac{ \tilde{ \tau}}{ \tau} \leq \min \{ \tfrac{2 \tau}{ \Lambda }, \max \{ \tfrac{ \tau^{2} }{\Lambda \abs{ \tau- \Lambda }},4 \} \}.
\]

The stable regions of the fixed points for $q=r=0$ and $ \abs{ P}<1$
can be derived in a similar way, which turns out to be a subset of
the stable region for $ P= \pm1$.

\subsubsection*{Type (4) fixed point at $ \protect P=r=0$ and $q \protect \neq0$.}
 From (\ref{eq:fixed-eq}), we know when at fixed point, $\tilde{q}=r=0$, then $q^2 = \frac{\Lambda -\tau}{\Lambda (1+\tau/2)}$, so $\tau$ must satisfy $\tau\leq \Lambda $. The corresponding Jacobian is:
\[
J(0,0,q)= \begin{bmatrix}0 & 0 & \ttau \Lambda  q \\
0 & \tau( \Lambda  - \tau)-3\tau\Lambda q^2(1+\frac{\tau}{2}) & 0 \\
\tau \Lambda  q & 0 & (\ttau-\tau)\Lambda  - \tau^2 - \tau\Lambda q^2(1+\frac{\tau}{2})
\end{bmatrix}.
\]
After plugging in $q^2 = \frac{\Lambda -\tau}{\Lambda (1+\tau/2)}$, we can obtain that the characteristic function $\det(\lambda \mI - J(0,0,q))$ is equal to:
\[
\det(\lambda \mI - J(0,0,q)) = [\lambda+2\tau(\Lambda -\tau)][\lambda(\lambda+(2\tau-\ttau)\Lambda )-\tau\ttau \Lambda^2 q^2]
\]
Clearly, $\det(\lambda \mI - J(0,0,q))=0$ has a non-negative root, so $J(0,0,q)$ always has a non-negative eigenvalue. This means type (4) fixed points are always unstable.

\subsubsection*{Type (5) fixed points at $ P,q,r \protect \neq0$}

The fixed points equation (\ref{eq:fixed-eq}) can also have solutions
that none of $ P$, $q$ and $r$ is zero. In what follows, we derive
the analytical expression of this type of solutions. It turns out
that there can be maximum 8 solutions, which are symmetric by flipping
the signs. We are unable to derive the analytical expression for their
stable region, but it can be computed numerically.

If $ P,q,r \neq0$, (\ref{eq:fixed-eq}) yields
\begin{align}
r & = \tfrac{ q}{P} \label{eq:fixed-r} \\
\Lambda - \tau- \Lambda (1+ \tfrac{ \tau}{2})q^{2}- \Lambda [ \tfrac{ P}{q}+( \tfrac{ \tau}{2}-1)r]r & =0 \label{eq:fixed-q} \\
\tau \Lambda  \t Pq+r \big[ \Lambda ( \ttau- \tau)- \tau^{2} \big]+r^{3} \Lambda  \big( \tau- \ttau- \tfrac{ \tau^{2}}{2} \big)-rq^{2} \tau \Lambda  \big(1+ \tfrac{ \tau}{2} \big) & =0. \label{eq:fixed-qtilde}
\end{align}
Plugging (\ref{eq:fixed-r}) into (\ref{eq:fixed-q}), we can get
\begin{equation}
q^{-2}=- \tfrac{1}{ \tau}[ \Lambda ( \tfrac{ \tau}{2}-1) P^{-2}+ \Lambda (1+ \tfrac{ \tau}{2})]. \label{eq:fixed-q1}
\end{equation}
Then combining (\ref{eq:fixed-r}) (\ref{eq:fixed-q1}) and (\ref{eq:fixed-qtilde}),
we can obtain the following equations:
\begin{equation}
A P^{-4}+B P^{-2}+C=0 \label{eq:q_tilde_equation}
\end{equation}
where $A= \Lambda ( \t{ \tau}- \tau)( \frac{1}{2}- \frac{1}{ \tau})+ \t{ \tau}$,
$B= \Lambda [ \frac{ \t{ \tau}}{ \tau}(1+ \frac{ \tau}{2})-2]$, $C= \Lambda (1+ \frac{ \tau}{2})$.
We can find that (\ref{eq:q_tilde_equation}) is an equation of $ P^{-2}$
with at most two roots. Combining (\ref{eq:fixed-q1}), we know there
are at most 2 solutions for the pair $(q^{-2}, \, P^{-2})$ and hence
there are at most 8 solutions for $(q, \, P,r)$, where $r= P/q$.
\subsection{Proof of Claim \ref{claim:stb}}
\begin{proof}[Proof of Claim \ref{claim:stb}]
We first compute
the Jacobian $
{\partial \big \{ \tfrac{ \dif}{\dif t} \mP_t, \tfrac{\dif}{\dif t} \vq_t, \tfrac{\dif}{\dif t} \vr_t, \big\} }/{ \partial\{ \mP_t, \vq_t, \vr_t \}}
$   of the ODE \eqref{eq:PCA-ODE} 
when  $\vq_t=\vr_t=\mat{0}$.
In the Jacobian, the  $d\times d$ matrix $\mP_t$ is considered as a $d^2$ vector. In fact, all elements in the Jacobian matrix related to $\mP_t$ are 0.
Specifically, the Jacobian for any $\mP$ and $\vq_t=\vr_t=\mat{0}$ is
\begin{equation} \label{eq:J}
\mJ(\mP)=
\begin{bmatrix}
\mat{0} & \mat{0} &\mat{0} \\
\mat{0} & \tau(\mat\Lambda - \tau \overline{\eta^2} \mI_d) & - \tau \mP \t{\mat\Lambda}  \\
\mat{0} &\tau \mP^\T \mat\Lambda & \t{\mat\Lambda}(\ttau - \tau) - \tau^2 \overline{\eta^2}
\end{bmatrix},
\end{equation}
where $ \overline{\eta^2}=(\eta_\text{T}^2 + \eta_\text{G}^2)/2$. 

When $\mP$ is diagonal, under a suitable column-row permutation, the $\mJ(\mP)$ in \eqref{eq:J} becomes a block diagonal matrix, where each non-zero block is a $2\times2$ matrix
\begin{equation} \label{eq:JPP}
\begin{bmatrix}
\tau([\mat{\Lambda}]_{\ell,\ell} - \tau \overline{\eta^2}) & -\tau [\mP]_{\ell,\ell} [\t{\mat{\Lambda}}]_{\ell,\ell}\\
\tau [\mP]_{\ell,\ell} [{\mat{\Lambda}}]_{\ell,\ell}
& [\t{\mat{\Lambda}}]_{\ell,\ell} (\ttau-\tau)-\tau^2 \overline{\eta^2}
\end{bmatrix}
\end{equation}
for $\ell=1,2,\ldots,d$. 
Intuitively, the above matrix is the Jacobian matrix of 
$
{\partial \{ \frac{\dif}{\dif t}[\vq_t]_\ell,\; \frac{\dif}{\dif t}[\vr_t]_\ell \}}/{\partial \{ [\vq_t]_\ell,\; [\vr_t]_\ell\}}
$, and the Jacobian 
$
{\partial \{ \frac{\dif}{\dif t}[\vq_t]_\ell,\; \frac{\dif}{\dif t}[\vr_t]_\ell \}}/{\partial \{ [\vq_t]_{\ell^\prime},\; [\vr_t]_{\ell^\prime}\}}
$ is zero for $\ell \neq \ell^\prime$.

Now the problem reduces into investigate eigenvalues of $n$ $2$-by-$2$ matrices.
For any given $\ell=1,2,\ldots,n$, we have studied this problem in Section \ref{sec:fiexed-d-1}
(type (1) and type (3) fixed points).

Specifically, the perfect recovery point $\mP=\mI$, $\vq=\vr=\mat{0}$ is stable if and only if
$\lambda_{\max}(\mJ(\mP))\leq0$, where $\mJ(\mP)$ is defined in \eqref{eq:J}. 
Similar to the analysis of the type (3) fixed points in Section \ref{sec:fiexed-d-1}, 
the condition that both eigenvalues of the matrix in \eqref{eq:JPP} is non-positive implies
\begin{align}
\tfrac{1}{2}([\mat{\Lambda}]_{\ell,\ell} - [\t{\mat{\Lambda}}]_{\ell,\ell} + \alpha [\t{\mat{\Lambda}}]_{\ell,\ell} ) &\leq \tau\overline{\eta^2} \label{eq:bb}
\\ \label{eq:alpha-1}
\text{ and } \quad
\a(\tau\overline{\eta^2} - [\mat{\Lambda}]_{\ell,\ell}) &\leq 
\tfrac{\tau\overline{\eta^2}}{[\t{\mat{\Lambda}}]_{\ell,\ell}} (\tau\overline{\eta^2} - [\mat{\Lambda}]_{\ell,\ell} + [\t{\mat{\Lambda}}]_{\ell,\ell}),
\end{align}
for all $\ell = 1,2,\ldots, n$. 
The inequality \eqref{eq:bb} is the first inequality of \eqref{eq:con} in  Claim 1 in the main text.

Next, we investigate the condition when the trivial fixed point of the origin $\mP=\mat{0}$ and $\vq=\vr=\mat{0}$ is unstable. Put $\mP=\mat{0}$ into \eqref{eq:JPP}, we get a diagonal matrix
\begin{equation*} 
\begin{bmatrix}
\tau([\mat{\Lambda}]_{\ell,\ell} - \tau \overline{\eta^2}) & 0\\
0
& [\t{\mat{\Lambda}}]_{\ell,\ell} (\ttau-\tau)-\tau^2 \overline{\eta^2}
\end{bmatrix}.
\end{equation*}
When any eigenvalue of the above matrices for $\ell = 1,2,\ldots,n$  is positive, this trivial fixed point will be unstable. A sufficient condition is the first eigenvalues of all matrices are positive:
\begin{equation} \label{eq:las}
 \tau \overline{\eta^2} < [\mat{\Lambda}]_{\ell,\ell}  
 \text{ for all } \ell = 1,2,\ldots,n.
\end{equation}
The above inequality is the second inequality of \eqref{eq:con} in the main text. 
In addition, \eqref{eq:las} implies \eqref{eq:alpha-1} hold as the left hand side of 
\eqref{eq:alpha-1} is negative. Now, we prove that  \eqref{eq:con}  is a sufficient condition that
the perfect fixed point is stable and the trivial fixed point is unstable.

\end{proof}

We further note that $\eqref{eq:con} $ is not a necessary condition. There may be a region that \eqref{eq:las} does not hold, but  the origin is still unstable, and the perfect recovery point is stable. Such region is hard to characterize analytically, and numerically, we found the training algorithms always converge to other bad fixed points (e.g. mode collapsing state, or a state that $\mP$ and $\vq$ are still zero, but $\vr$ is non-zero. The situation of the latter is similar to the noninfo-2 phase in the $d=1$ case, which converges to the type (2) fixed point). Further study on those bad fixed points will be established in future works under a more general model.  


\end{document}